\documentclass{new_tlp}
\usepackage{mathptmx}

\usepackage{times}
\usepackage{helvet}
\usepackage{courier}

\usepackage{graphicx}
\usepackage{latexsym}

\usepackage{url}
\usepackage{xspace}

\usepackage{amsmath}
\usepackage{amsfonts}
\usepackage{amssymb}
\usepackage{listings}
\usepackage{verbatim}

\usepackage{todonotes}

\newcommand{\adfnode}[5]{
\node[afnode] (#3) #2 {#4};
\node[#1 of=#3, node distance=5.5mm] {#5};
}

\tikzstyle{afnode} = [draw,thick,shape=circle,minimum size=6mm,font=\normalsize]
\tikzstyle{adflink} = [->,draw,thick]

\newcommand{\asp}{ASP\xspace}

\newcommand{\adf}{ADF\xspace}
\newcommand{\adfs}{ADFs\xspace}
\newcommand{\grappa}{GRAPPA\xspace}

\newcommand{\pname}[1]{\texttt{#1}\xspace}
\newcommand{\diam}{\pname{DIAMOND}}
\newcommand{\qadf}{\pname{QADF}}
\newcommand{\adm}{\textit{adm}}
\newcommand{\com}{\textit{com}}
\newcommand{\prf}{\textit{prf}}
\newcommand{\pef}{\prf}
\newcommand{\ground}{\textit{grd}}
\newcommand{\grd}{\ground}
\newcommand{\stb}{\textit{stb}}
\newcommand{\stable}{\stb}
\newcommand{\mdl}{\textit{mod}}
\newcommand{\truthvals}{\{\tvt,\tvf,\tvu\}}

\newcommand{\lag}{LAG\xspace}

\newcommand{\PNAME}[1]{\texttt{#1}}
\newcommand{\nop}[1]{}

\newcommand{\myIndex}[1]{\index{#1}} %
\newcommand{\tvt}{1} %
\newcommand{\tvf}{0} %
\newcommand{\tvu}{\ensuremath{\textbf{u}}\xspace} %

\newcommand{\ac}[1]{\varphi_{#1}}

\newcommand{\set}[1]{\ensuremath{\left\{#1\right\}}}

\newcommand{\comp}[1]{\ensuremath{\lbrack #1 \rbrack_c}}

\newcommand{\mtt}[1]{\text{\texttt{#1}}}

\newcommand{\ccfont}[1]{\protect\mathsf{#1}}
\newcommand{\NP}{\ensuremath{\ccfont{NP}}\xspace}
\newcommand{\PP}{\ensuremath{\ccfont{P}}\xspace}
\newcommand{\NEXP}{\ensuremath{\ccfont{NEXP}}\xspace}
\newcommand{\Ptime}{\ccfont{P}}
\newcommand{\co}{\ensuremath{\ccfont{co}}}
\newcommand{\phs}[1]{\ensuremath{\Sigma_{#1}^\Ptime}}

\newcommand{\php}[1]{\ensuremath{\Pi_{#1}^\Ptime}}

\newcommand{\ileq}{\leq_i}
\newcommand{\ilt}{<_i}

\newcommand{\parent}{\textit{par}}

\newcommand{\prts}[1]{\mathit{par}_{D}(#1)}

\newcommand{\ie}{i.e.\xspace}
\newcommand{\wrt}{w.r.t.\xspace}
\newcommand{\resp}{resp.\xspace}

\newcommand{\txt}[1]{\textit{#1}}

\newcommand{\cpl}[1]{[#1]_2}
\newcommand{\ext}[1]{\cpl{#1}}
\newcommand{\extension}[1]{\cpl{#1}}

\newcommand{\tvtt}{1}
\newcommand{\tvft}{0}
\newcommand{\tvut}{\textbf{u}}

\newcommand{\leftarr}{\derive}

\newcommand{\adfints}{V}
\newcommand{\adfint}{v}
\newcommand{\aspints}{\mathcal{I}}
\newcommand{\aspint}{I}
\newcommand{\AS}{\mathcal{A\hspace{-1mm}S\hspace{-0.45mm}}}
\newcommand{\agrd}[1]{Gr(#1)}
\newcommand{\ared}[2]{#1^{#2}}

\newcommand{\projection}[2]{{#1}\downharpoonright_{#2}}

\newcommand{\satint}[1]{#1^{\triangle}}

\newcommand{\PassP}{\ensuremath{\textit{asg}}}
\newcommand{\PargP}{\ensuremath{\textit{arg}}}
\newcommand{\PleqP}{\ensuremath{\textit{lt}}}
\newcommand{\PsatP}{\ensuremath{\textit{sat}}}
\newcommand{\PunsatP}{\ensuremath{\textit{inv}}}

\newcommand{\PassTP}{\ensuremath{\textit{asg2}}}
\newcommand{\PsatTP}{\ensuremath{\textit{sat2}}}
\newcommand{\PunsatTP}{\ensuremath{\textit{inv2}}}
\newcommand{\PsaturateP}{\ensuremath{\textit{saturate}}}

\newcommand{\PcmP}{\ensuremath{\textit{cm}}}

\newcommand{\PlneP}{\ensuremath{\textit{lne}}}
\newcommand{\PleqSP}{\ensuremath{\textit{lt2}}}
\newcommand{\PpropP}{\ensuremath{\textit{prop}}}
\newcommand{\PpropSP}{\ensuremath{\textit{prop2}}}

\newcommand{\Pass}[1]{\ensuremath{\textit{asg}(#1)}}
\newcommand{\Parg}[1]{\ensuremath{\textit{arg}(#1)}}
\newcommand{\Pleq}[1]{\ensuremath{\textit{lt}(#1)}}
\newcommand{\Psat}[1]{\ensuremath{\textit{sat}(#1)}}
\newcommand{\Punsat}[1]{\ensuremath{\textit{inv}(#1)}}
\newcommand{\Pnot}{\ensuremath{\textit{not}}}
\newcommand{\derive}{\ensuremath{\texttt{:-}}}
\newcommand{\PassT}[1]{\ensuremath{\textit{asg2}(#1)}}
\newcommand{\PsatT}[1]{\ensuremath{\textit{sat2}(#1)}}
\newcommand{\PunsatT}[1]{\ensuremath{\textit{inv2}(#1)}}
\newcommand{\Psaturate}{\ensuremath{\textit{saturate}}}

\newcommand{\Plne}[1]{\PlneP(#1)}
\newcommand{\Pprop}[1]{\PpropP(#1)}

\newcommand{\Pcm}{\PcmP}

\newcommand{\PleqS}[1]{\PleqSP(#1)}
\newcommand{\PpropS}[1]{\PpropSP(#1)}

\newcommand{\B}{\Omega}
\newcommand{\BT}[1]{\B^{#1}}
\newcommand{\Bl}{\omega}

\newcommand{\Ml}{\mu}
\newcommand{\LE}{\lambda}

\newcommand{\CMl}{\kappa}

\newcommand{\aneq}{\neq}

\def\Lplus{\texttt{+}}
\def\Lminus{\texttt{-}}

\newtheorem{defn}{Definition}
\newtheorem{propn}{Proposition}
\newtheorem{theorem}{Theorem}

\newtheorem{example}{Example}

  \title[Solving Advanced Argumentation Problems with ASP]
        {Solving Advanced Argumentation Problems with Answer Set Programming\footnote{
This research has been supported
by
DFG
(projects BR 1817/7-2 as well as 389792660 - TRR 248) and
FWF
(projects I2854, Y698, P32830, S11409-N23, and W1255-N23).
The authors also thank J\"org P\"uhrer for his helpful observations regarding the encodings.  Also thanks to Wolfgang Dvo{\v{r}}{\'{a}}k and Ringo Baumann for guiding the paper through the editing process. 
}}
\author[Gerhard Brewka et. al.]
{Gerhard Brewka\\
  Universit\"at Leipzig, Leipzig, Germany
\and Martin Diller\\
TU Dresden, Dresden, Germany  
\and Georg Heissenberger, Thomas Linsbichler, Stefan Woltran\\
TU Wien, Vienna, Austria}
        
\pagerange{\pageref{firstpage}--\pageref{lastpage}}
\begin{document}

\label{firstpage}

\maketitle
\begin{abstract}
Powerful formalisms for abstract argumentation
have been proposed, among them 
abstract dialectical frameworks (ADFs) that allow
for a succinct and flexible specification of the 
relationship between arguments,
and the GRAPPA framework which 
allows argumentation scenarios to be represented as arbitrary edge-labelled graphs.
%
%Depending on the semantics, 
The complexity of ADFs and GRAPPA is located beyond NP and
ranges up to the third level of the polynomial hierarchy.
The combined complexity of Answer Set Programming (ASP)
exactly matches this complexity when programs are restricted
to predicates of bounded arity. In this paper, we exploit
this coincidence and present novel efficient translations from
ADFs
and GRAPPA to ASP.  
More specifically, we provide reductions for the five main
ADF semantics of admissible, complete, preferred, grounded, and
stable interpretations, and exemplify how these reductions need
to be adapted for GRAPPA for the admissible, complete and preferred semantics.  
Under consideration in Theory and Practice of Logic Programming (TPLP).
%We also empirically compare our approach
%to other systems for ADF reasoning and report promising results.
\end{abstract}


  \tableofcontents

  \section{Introduction}

Argumentation is an active area of research with applications in legal reasoning \cite{Bench-CaponD05}, decision making \cite{AmgoudP09}, e-governance \cite{CartwrightA09} and multi-agent systems \cite{ConfArgmas2011}.
Dung's
argumentation frameworks \cite{DBLP:journals/ai/Dung95}, AFs for short, are widely used in argumentation.
They focus entirely on conflict resolution among arguments, treating the latter as abstract items without
logical structure. Although AFs are quite popular,
various generalizations aiming for easier and more natural representations %
have been proposed; see \cite{DBLP:journals/expert/BrewkaPW14} for an overview.

We focus on two such generalizations, namely
\adfs \cite{DBLP:conf/kr/BrewkaW10,DBLP:conf/ijcai/BrewkaSEWW13} and
\grappa \cite{DBLP:conf/ecai/BrewkaW14},
which are expressive enough to capture many of the other available frameworks;
see also the recent handbook article~\cite{BrewkaESWW18} which surveys both formalisms.
Reasoning in \adfs
spans the first three levels
of the polynomial hierarchy \cite{DBLP:journals/ai/StrassW15}.
These results carry over to \grappa ~\cite{DBLP:conf/ecai/BrewkaW14}.
ADFs, in particular, have received increasing attention recently,
see e.g.\
\cite{GagglS14,Booth2015} including also practical applications 
in fields such as 
legal reasoning~\cite{Al-Abdulkarim2016,AtkinsonB18},
text exploration~\cite{CabrioV16} or
discourse analysis~\cite{Neugebauer18}.

Two approaches to implement ADF reasoning
have been proposed in the literature.
\qadf~\cite{DBLP:conf/comma/DillerWW14,DilWalWol2015} encodes problems as quantified Boolean formulas (QBFs)
such that a single call of a QBF solver delivers the result. 
The
\diam family of systems~\cite{EllmauthalerS2014,EllmauthalerS2016,StrassE17}, on the other hand, employ %
Answer Set Programming (ASP).
Since the \diam systems rely on \emph{static} encodings, \ie the encoding does not change for different framework instances,
this approach is limited by the \emph{data complexity} of ASP
(which only reaches the second level
of the polynomial hierarchy~\cite{DBLP:journals/amai/EiterG95,DBLP:journals/tods/EiterGM97}).
Therefore, the preferred semantics in particular
(which comprise the hardest problems
for ADFs and GRAPPA)
needs a more complicated treatment
involving two consecutive calls to ASP solvers
with a possibly exponential blowup for the input of the second call.
A GRAPPA interface has been added to \diam  \cite{Berthold16},
but
we are not aware of any systems for \grappa
not employing a translation to ADFs as an intermediate step.

In this paper,
we introduce a new method for
implementing
reasoning tasks related to both ADFs and GRAPPA
such that even the hardest among the problems
are treated with a \emph{single} call to an ASP solver
(and avoiding any exponential blow-up in data or program size). 
The reason for choosing ASP is that the rich syntax of \grappa
is captured much more easily by ASP than by other formalisms
like QBFs.
Our approach makes use of the fact that
the \emph{combined complexity} of ASP
for programs %
with predicates
of bounded arity~\cite{DBLP:journals/amai/EiterFFW07}
exactly matches the complexity of ADFs and GRAPPA.
This approach is called \emph{dynamic},
because the encodings are generated individually for every instance.
This allows to generate rules of arbitrary length that can take care
of NP-hard subtasks themselves.
This particular method has been advocated in
\cite{BichlerMW2016} in combination with tools that decompose such 
long rules whenever possible in order to avoid huge groundings 
\cite{BichlerMW16c}.
To the best of our knowledge, our work is the first to apply this 
technique in the field of argumentation.
%We are not aware of any comparable work
%in %
%kkkargumentation.

More specifically,
we provide encodings
for the admissible, complete, preferred, grounded and stable semantics
for \adfs  
and discuss how such encodings can be adapted to \grappa.
Depending on the semantics
(and their complexity)
the encodings yield normal or,
in the case of preferred semantics,
disjunctive programs.
We specify the encodings in a modular way,
which makes our approach amenable for extensions
to other semantics.
We further provide some details about the resulting
system \PNAME{YADF} (``Y'' standing for dYnamic) which is publicly available 
%ADF and GRAPPA\footnote{The system for ADF reasoning is available 
at \url{https://www.dbai.tuwien.ac.at/proj/adf/yadf/}.  Finally, we give an overview of recent empirical evaluations, including our own, comparing the performance of \PNAME{YADF} with the other main existing ADF systems.%\todo{MD@SW: ok to add this last sentence? SW: Yes}.  
%We also provide a preliminary experimental analysis
%in order to determine the potential of this method.
%
%
%
%
%

This paper is an extended version of \cite{BrewkaDHLW17}, which did not
contain the encodings for the grounded and stable semantics. In addition,
we provide some prototypical proofs for the correctness of the encodings.  We also update the discussion on empirical evaluations. The paper is based on (Section 3.2 of) the second author's thesis~\cite{dil19}.%\todo{MD@SW: idem}.  
%\todo{SW: shall we add you phd thesis here; at least we should do so in footnote 11}

\section{Background}

\paragraph{ADFs.}
An \adf %
is a directed graph whose nodes represent statements. The links represent dependencies: the acceptance status of a node $s$ only depends on the acceptance status of its parents (denoted $par(s)$; often also with a subscript as in $\prts{s}$ to make the reference to the ADF $D$ explicit), that is, the nodes with a direct link to $s$. In addition, each node $s$ has an associated acceptance condition $C_s$ specifying the %
conditions under which $s$ is acceptable.

%

\begin{comment}
\begin{definition}
An \emph{abstract dialectical framework} is a tuple $D = (S, L,C)$ where
\begin{itemize}
\item $S$ is a set of statements,
\item $L \subseteq S \times S$ is a set of edges,
\item $C = \set{\ac_s}_{s\in S}$ is a set of propositional functions over $par(s)$, one for each statement~$s$.
%
%
$\ac_s$ is called the \emph{acceptance condition} of~$s$.
\end{itemize}
\end{definition}
\end{comment}
It is convenient to represent the acceptance conditions as a collection $C = \{\ac{s}\}_{s \in S}$ of propositional formulas. This leads to the logical representation of \adfs we will use in this paper where an \adf $D$ is a pair $(S,C)$ with the set of links $L$ implicitly given as $(a, b) \in L$ iff $a$ appears in $\ac{b}$.

Semantics assign to ADFs a collection of (3-valued) interpretations, i.e. mappings of the statements to truth values $\{\tvt,\tvf,\tvu\}$, denoting true, false and undecided, respectively.
The three truth values are partially ordered by $\ileq$ according to their information content:
we have \mbox{$\tvu \ilt \tvt$} and \mbox{$\tvu \ilt \tvf$} and no other pair in $\ilt$.
The information ordering $\ileq$ extends in a straightforward way to interpretations $v_1,v_2$ over $S$ in that
\mbox{$v_1 \ileq v_2$} iff \mbox{$v_1(s) \ileq v_2(s)$} for all \mbox{$s\in S$}.

An
interpretation $v$ is 2-valued if all statements are mapped to $\tvt$ or $\tvf$.
For %
interpretations $v$ and $w$, we say that %
$w$ extends $v$ iff \mbox{$v \ileq w$}.
We denote by $[v]_2$ the set of all completions of $v$, i.e.\ 2-valued interpretations that extend~$v$.

For an ADF $D=(S,C)$, $s\in S$ and an
interpretation $v$, the characteristic function $\Gamma_D(v) = v'$ is given by
\[
  v'(s) = \begin{cases}
          \tvt \text{ if }  w(\ac{s}) = \tvt \text{ for all } w\in[v]_2\\
          \tvf \text{ if }  w(\ac{s}) = \tvf \text{ for all } w\in[v]_2\\
          \tvu \text{ otherwise}.
          \end{cases}
\]

\noindent That is, the operator returns an
interpretation mapping a statement $s$
to $\tvt$ (resp.\ $\tvf$) iff all 2-valued interpretations extending $v$ evaluate $\ac{s}$ to true (resp.\ false).
Intuitively, $\Gamma_D$ checks which truth values can be justified based on the information in $v$ and the acceptance conditions. Note that
$\Gamma_D$
is defined on 3-valued interpretations, while we evaluate acceptance conditions under their 2-valued completions.%
Given
an \adf $D=(S,\{\ac{s}\}_{s \in S})$,
an
interpretation $v$ is \emph{admissible} \wrt $D$ if
$v \ileq \Gamma_D(v)$; it
is \emph{complete} \wrt $D$ if
$v = \Gamma_D(v)$;
it is \emph{preferred} \wrt $D$ if $v$ is maximal admissible \wrt $\ileq$.
An interpretation $v$ is the (unique) \emph{grounded} interpretation \wrt $D$ if $v$ is complete and there is no other complete interpretation $w$ for which $w \ilt v$.

Turning to semantics returning 2-valued interpretations, a 2-valued interpretation $v$ is a \emph{model} of $D$ if $v(s) = v(\ac{s})$ for every $s \in S$. The definition of the stable semantics for ADFs is inspired by the stable semantics for logic programs, its purpose being to disallow cyclic support within a model.  First of all, in order to be a \emph{stable} model of $D$ $v$ needs to be a model of $D$.  Secondly, $E_v = \{s \in S \mid v(s) = \tvt\}$ must equal the statements set to true in the grounded interpretation of the reduced ADF \mbox{$D^v = (E_v, \{\ac{s}^v\}_{s \in E_v})$}, where
%  \mbox{$L^v = L \cap (E_v\times E_v)$} and
  for \mbox{$s \in E_v$} we set
  \mbox{$\ac{s}^v := \ac{s}[b/\bot:v(b)=\tvf]$}. If $\projection{v}{E_v}$ is the interpretation $v$ projected on $E_v$, i.e. $\projection{v}{E_v}(s) = v(s)$ for $s \in E_v$ and undefined otherwise, then the latter amounts to the fact that $\projection{v}{E_v}$ be the grounded interpretation of $D^v$.

As shown in \cite{DBLP:conf/ijcai/BrewkaSEWW13}
these semantics generalize the corresponding notions defined %
for AFs.
For $\sigma \in \{\adm,\com,\prf,\ground, \mdl, \stable \}$, $\sigma(D)$ denotes the set of all
admissible (resp.\ complete, preferred, grounded, model, stable) interpretations \wrt $D$.

%
%\begin{example}
%\label{ex:adf}
%Given the \adf $D = (\{a,b\},\{\ac{a} {=} a \vee \neg b, \ac{b} {=} \neg a\})$
%and
%
%$v_1 = \{a {\rightarrow} \tvu, b {\rightarrow} \tvu\}$,
%$v_2 = \{a {\rightarrow} \tvt, b {\rightarrow} \tvu\}$,
%$v_3 = \{a {\rightarrow} \tvt, b {\rightarrow} \tvf\}$,
%$v_4 = \{a {\rightarrow} \tvf, b {\rightarrow} \tvt\}$,
%we get
%$\adm(D) = \{v_1,v_2,v_3,v_4\}$,
%$\com(D) = \{v_1,v_3,v_4\}$
%(note that $\Gamma_D(v_2){=}v_3$, thus $v_2 {\notin} \com(D)$),
%and
%$\prf(D) = \{v_3,v_4\}$.
%\end{example}

\begin{example}
\label{ex:adf}
In Figure~\ref{fig:adf_example} we see an example ADF $D=(\{a,b,c\},C)$ with the acceptance conditions $C$ given by $\varphi_a= b \lor \neg b$, $\varphi_b=b$ and $\varphi_c=c\rightarrow b$. The acceptance conditions are shown below the statements in the figure.

\begin{figure}
\begin{center}
\begin{tikzpicture}
\adfnode{below}{}{a}{$a$}{$b\lor \neg b$}
\adfnode{below}{at (2,0)}{b}{$b$}{$b$}
\adfnode{below}{at (4,0)}{c}{$c$}{$c \rightarrow b$}

\draw[adflink,<-] (a) -- (b);
\draw[adflink,->] (b) -- (c);
\draw[adflink,in=350,out=10,loop] (c) to (c);
\draw[adflink,in=100,out=80,loop] (b) to (b);

\end{tikzpicture}
\caption{ADF example}
\label{fig:adf_example}
\end{center}
\end{figure}
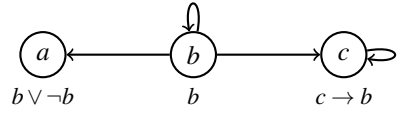

\begin{table}
  \centering
\begin{tabular}{c|ccc|c}
      & $a$    & $b$    & $c$    & \\
\hline
$v_1$ & $\tvu$ & $\tvu$ & $\tvu$ & $\adm$\\
$v_2$ & $\tvu$ & $\tvf$ & $\tvu$ & $\adm$\\
$v_3$ & $\tvu$ & $\tvt$ & $\tvu$ & $\adm$\\
$v_4$ & $\tvt$ & $\tvt$ & $\tvu$ & $\adm$\\
$v_5$ & $\tvu$ & $\tvt$ & $\tvt$ & $\adm$\\
$v_6$ & $\tvt$ & $\tvu$ & $\tvu$ & $\adm$, $\com$, $\ground$\\
$v_7$ & $\tvt$ & $\tvf$ & $\tvu$ & $\adm$, $\com$, $\pef$\\
$v_8$ & $\tvt$ & $\tvt$ & $\tvt$ & $\adm$, $\com$, $\pef$, $\mdl$
\end{tabular}
\caption{All admissible interpretations of the ADF from Figure~\ref{fig:adf_example}. The right most column shows further semantics the interpretations belong to. }
\label{tab:example}
\end{table}

The admissible interpretations of $D$ are shown in Table~\ref{tab:example}. Moreover, the right-most column shows further semantics the interpretations belong to. For instance the interpretation $v_8$ mapping each statement to true is admissible, complete and preferred in $D$ and a model of $D$. %This ADF has no stable models. 
	The only model of $D$ is $v_8$, with the reduct of this model being $D^{v_8} = D$. The grounded interpretation of $D$ is $v_6$, which is different from $v_8$. Therefore $v_8$ is not a stable model. 
In fact, $D$ does not have a stable model.
\hfill$\diamond$

% This ADF has three complete interpretations: $v_1$, $v_2$ and $v_3$ with $v_1(a)=\tvt$, $v_1(b)=v_1(c)=\tvu$, $v_2(a)=v_2(b)=v_2(c)=\tvt$ and $v_3(a)=\tvt$,  $v_3(b)=v_3(c)=\tvf$. Further $v_1$ is the grounded interpretation of $D$, both $v_2$ and $v_3$ are preferred interpretations and two valued models of $D$. Only $v_3$ is a stable model of $D$. %To see that $v_2$ is not stable, consider its reduct $D^{v_2}$, which is equal to $D$. Since the grounded interpretation of $D^{v_2}$ is not equal to $v_2$ we have $v_2 \notin \stable(D)$.
\end{example}

\paragraph{GRAPPA.}
\adfs are particularly useful as target formalism of translations from graph-based approaches. This raises the question whether an \adf style semantics can be directly defined for arbitrary labelled graphs, thus circumventing the need for any translations.
\grappa \cite{DBLP:conf/ecai/BrewkaW14} fulfills exactly that goal.

\grappa allows argumentation scenarios to be defined using arbitrary directed edge-labelled graphs. The nodes in $S$ represent statements, as before. Labels of links, which may be chosen as needed, describe the type of relationship between a node and its parents. As for \adfs, each node has its own acceptance condition, and the semantics of a graph is defined in terms of 3-valued interpretations. The major difference is that acceptance conditions are no longer specified in terms of the acceptance status of the parents of a node, but on the labels of its active incoming links, where a link is active if its source node is true and a label is active if it is the label of an active link. More precisely, since it can make an important difference whether a specific label appears once or more often on active links, the acceptance condition depends on the multiset of active labels of a node, that is, an acceptance condition is a function of the form $(L \rightarrow \mathbb{N}) \rightarrow \{\tvt,\tvf\}$, where $L$ is the set of all labels.

\grappa acceptance functions
are specified using
\emph{acceptance patterns} over a set of labels
$L$ %
defined
as follows:
\begin{itemize}
\item A \emph{term} over $L$ is of the form
$\#(l)$, $\#_t(l)$ %
(with $l \in L$), or
$min$, $min_t$, $max$, $max_t$, $sum$, $sum_t$, $count$, $count_t$.
\item A \emph{basic acceptance pattern} (over $L$) is of the form
$a_1 t_1 + \cdots + a_n t_n \,R\, a$,
	where the $t_i$ are terms over $L$, the $a_i$s and $a$ are integers and
	$R \in \{<,\leq, =, \neq, \geq, >\}$.
\item An \emph{acceptance pattern} (over $L$) is a basic acceptance pattern or a Boolean combination of acceptance patterns.
\end{itemize}

A \grappa \emph{instance} is a tuple $G = (S,E,L,\lambda,\alpha)$
where $S$ is a set of statements, $E$ a set of edges, $L$ a set of labels, $\lambda$ an assignment of labels to edges, and  $\alpha$ an assignment of acceptance patterns over $L$ to nodes.

For a multiset of labels $m: L \rightarrow \mathbb{N}$ and $s \in S$ the value function $val_s^m$ is:

\begin{tabbing}
xx \= xxxxxxxxxxxxxx \= \kill
\> $val_s^m(\#l) $ \> $ = m(l)$\\
\>$val_s^m(\#_tl) $ \> $ = |\{(e,s) \in E \mid \lambda((e,s)) = l\}|$ \\
\>$val_s^m(min) $ \> $ = \textbf{min} \{ l \in L \mid m(l) > 0 \}$ \\
\>$val_s^m(min_t) $ \> $ = \textbf{min} \{\lambda((e,s)) \mid (e,s) \in E \}$ \\
\>$val_s^m(max) $ \> $ = \textbf{max} \{ l \in L \mid m(l) > 0 \}$ \\
\>$val_s^m(max_t) $ \> $ = \textbf{max} \{\lambda((e,s)) \mid (e,s) \in E \}$ \\
\>$val_s^m(sum) $ \> $ = \sum_{l \in L} m(l)$ \\
\>$val_s^m(sum_t) $ \> $ = \sum_{(e,s) \in E} \lambda((e,s))$ \\
\>$val_s^m(count) $ \> $ = | \{ l \mid m(l)>0\}| $ \\
\>$val_s^m(count_t) $ \> $ = | \{ \lambda((e,s)) \mid (e,s) \in E \}|$ \\
\end{tabbing}

\noindent $min_{(t)}$, $max_{(t)}$, $sum_{(t)}$ are undefined in case of non-numerical labels.
For $\emptyset$
they yield the neutral element of the corresponding operation, \ie
$val_s^m(sum)=val_s^m(sum_t)=0$,
$val_s^m(min)=val_s^m(min_t)=\infty$, and
$val_s^m(max)=val_s^m(max_t)=-\infty$.  Let $m$ and $s$ be as before. %
For a basic acceptance pattern
%\begin{align*}
$\alpha = a_1 t_1 + \cdots + a_n t_n R$
%\end{align*}
%\noindent
we define $\alpha(m,s) = \tvt$ if $\sum_{i=1}^n \big(a_i \;val_s^m(t_i)\big) \;R \;a$, while $\alpha(m,s) = \tvf$ otherwise.  The extension to the evaluation of Boolean combinations is as usual.

%the \emph{satisfaction relation} $\models$ is defined by
%\[
%(m,s) \models a_1 t_1 + \cdots + a_n t_n R\, a \, \mbox{ \text{ iff } }\, \sum_{i=1}^n \big(a_i \;val_s^m(t_i)\big) \;R \;a.
%\]

%For each node $s$ the associated acceptance function $c(s)$ is obtained from the corresponding pattern $\pi(s)$ by
%$c(s)(m) = \tvt \mbox{ \text{ iff } } (m,s) \models \pi(s)$.

The characteristic function $\Gamma_G$ for a GRAPPA instance $G$, as is the case for the characteristic function for ADFs, takes a 3-valued interpretation $v$ and produces a new one $v'$. Again $v'$ is constructed by considering all 2-valued completions $w$ of $v$, picking a classical truth value only if all extensions produce the same result. But this time an intermediate step is needed to determine the truth value of a node $s$: one first has to determine the multiset of active labels of $s$ generated by $w$. The acceptance function then takes this multiset as argument and produces the truth value induced by $w$.

Let $v$ be a two valued interpretation.  The multi-set of active labels of $s\in S$ in $G$ under $v$, $m^{v}_s$, is defined as

\begin{align*}
  m^{v}_s(l)=|\{(p, s) \in E \mid v(p)=\tvt, \lambda((p, s))=l\}|
\end{align*}

\noindent for each $l\in L$.  Then the characteristic function $\Gamma_G(v) = v'$ for a GRAPPA instance $G$ is given by
\[
  v'(s) = \begin{cases}
          \tvt \text{ if }  \alpha(s)(m^{w}_s,s) = \tvt \text{ for all } w\in[v]_2\\
          \tvf \text{ if }  \alpha(s)(m^{w}_s,s) = \tvf \text{ for all } w\in[v]_2\\
          \tvu \text{ otherwise}.
          \end{cases}
\]

\noindent With this new characteristic function the semantics of a graph $G$ can be defined as for \adfs, that is,
an
interpretation $v$ is \emph{admissible} \wrt $G$ if
$v \ileq \Gamma_G(v)$; it
is \emph{complete} \wrt $G$ if
$v = \Gamma_G(v)$;
it is \emph{preferred} \wrt $G$ if $v$ is maximal admissible \wrt $\ileq$.
As before $\sigma(G)$ ($\sigma \in \{\adm,\com,\prf\}$) denotes the set of all respective interpretations.

\begin{example}
\label{ex:grappa}
Consider the GRAPPA instance $G$
with
$S=\{a,b,c\}$,
$E=\{(a,b),(b,b),(c,b),(b,c)\}$,
$L=\{\Lplus,\Lminus\}$,
all edges being labelled with $\Lplus$ except $(b,b)$ with $\Lminus$,
and
the
acceptance condition
$\#_t(\Lplus) - \#(\Lplus) = 0 \wedge \#(\Lminus) = 0$
(i.e.\ all $\Lplus$-links must be active and no $\Lminus$-link is active)
for each statement.
The following interpretations are admissible \wrt $G$:
$v_1 = \{a {\rightarrow} \tvu, b {\rightarrow} \tvu, c {\rightarrow} \tvu\}$,
$v_2 = \{a {\rightarrow} \tvu, b {\rightarrow} \tvf, c {\rightarrow} \tvf\}$,
$v_3 = \{a {\rightarrow} \tvt, b {\rightarrow} \tvu, c {\rightarrow} \tvu\}$,
$v_4 = \{a {\rightarrow} \tvt, b {\rightarrow} \tvf, c {\rightarrow} \tvf\}$.
Moreover,
$\com(G)=\{v_3,v_4\}$ and $\prf(G)=\{v_4\}$.
\hfill$\diamond$
\end{example}

\begin{comment}
Now the characteristic operator for a \grappa instance can be defined equivalently to \lag.

\begin{definition}\label{def:charGrappa}
Let $G=(S,E,L,\lambda,\pi)$ be a \grappa instance.
The \emph{characteristic operator} $\Gamma_G(v) = P_G(v) \cup N_G(v)$ with
\begin{align*}
P_G(v) &= \set{s \mid (m,s)\models\pi(s) \mbox{ for each } m \in \set{m_s^{v'} \mid v' \in \comp v }}, \\
N_G(v) &= \set{\neg s \mid (m,s) \not\models\pi(s) \mbox{ for each } m \in \set{m_s^{v'} \mid v' \in \comp v }}.
\end{align*}
\end{definition}

With the operator $\Gamma_G$ as defined in Definition \ref{def:charGrappa} the semantics of \lag from Definition \ref{def:lagSem} can be carried over to \grappa instances.

%
%
%
%
%
%
%
%
%
%

%
%
%
%
%
%
%

\end{comment}

%
\paragraph{ASP.}
In Answer Set Programming
\cite{DBLP:journals/tocl/LeonePFEGPS06,BrewkaET11}
problems are described using logic programs, which are sets of rules of the form

\[
a_1 \vee \ldots \vee a_n \derive b_1 ,\ldots, b_k , \textit{ not } b_{k+1} , \ldots , \textit{ not } b_m.
\]

\noindent
Here each
$a_i$ ($1 \leq i \leq n$) and $b_j$ ($1 \leq j \leq m$) is a ground atom.  The symbol $not$ stands for \emph{default negation}.
We call a rule a \emph{fact} if $n=0$. %
An \emph{(input) database} is a set of facts.
A rule $r$ is \emph{normal} if $n \leq 1$ and a \emph{constraint} if $n = 0$.  $B(r)$ denotes the body of a rule and $H(r)$ the head.  
A \emph{program} is a finite set of %
disjunctive rules.
If each rule in a program is normal %
we call the program normal, otherwise the program is disjunctive.

Each logic program $\pi$ induces a collection of so-called \emph{answer sets},
denoted as $\AS(\pi)$,
which are distinguished models of the program determined by the answer set semantics.
The answer sets of a program $\pi$ are the subset-minimal models satisfying the Gelfond-Lifschitz reduct $\ared{\pi}{\aspint}$ of $\pi$; see \cite{DBLP:journals/ngc/GelfondL91} for details.  

For non-ground programs, which we use here,
rules with variables are viewed as shorthand for the set of their ground instances. We denote by $\agrd{\pi}$ the ground instance of a program $\pi$.
Modern ASP solvers %
offer further additional language features such as built-in arithmetics and aggregates
which we make use of in our encodings
(we refer to \cite{web:potguide} for an explanation).

\begin{table}
\begin{center}
\caption{Complexity results for \adfs, \grappa and ASP.}
\label{tbl:complResGrappa}
\vspace{-2.5mm}
\begin{small}
\begin{tabular}{l | ccccc | cc }
\hline
& \multicolumn{5}{c|}{\adf and \grappa} & \multicolumn{2}{c}{ASP bounded arity} \\
         & \adm & \com & \prf & \ground & \stb & normal & disjunctive \\
\hline			
    cred & \phs{2} & \phs{2} & \phs{2} & \co\NP & \phs{2} & \phs{2} & \phs{3} \\ 			
    skept & trivial & \co\NP & \php{3} & \co\NP & \php{2} & \php{2} & \php{3}
\end{tabular}
\end{small}
\end{center}
\end{table}

\paragraph{Complexity.}
The
complexity results %
that are central for our work
are given in Table \ref{tbl:complResGrappa}.
Here credulous reasoning means deciding whether a statement (resp.\ atom) is true in at least one interpretation (resp.\ answer set) of the respective type,
skeptical reasoning whether it is true in all such interpretations (resp.\ answer sets).

The results %
for \adfs \cite{DBLP:journals/ai/StrassW15} carry over to \grappa, as argued in~\cite{DBLP:conf/ecai/BrewkaW14}.
The results for normal and disjunctive \asp-programs we use here refer to the combined complexity for
non-ground programs of \emph{bounded} predicate arity
(\ie there exists a constant $n \in \mathbb{N}$ such that the arity of every predicate occurring in the program is smaller than $n$) and
are due to \cite{DBLP:journals/amai/EiterFFW07}.
We recall that the combined complexity of arbitrary %
programs
is much higher ($\NEXP$-hard, see e.g. \cite{DBLP:journals/tods/EiterGM97})
while data complexity
(\ie the \asp-program is assumed to be static and only the database of the program is changing) is one level lower in the polynomial hierarchy
(follows from
\cite{DBLP:journals/amai/EiterG95}).

These results indicate that there exist efficient translations to non-ground normal programs
of bounded arity for credulous reasoning \wrt the admissible, complete, preferred, and stable semantics; skeptical reasoning for the stable semantics can be reduced to skeptical reasoning for normal programs.  Skeptical preferred reasoning needs to be
treated with disjunctive programs. We provide such reductions in
what follows.

\section{ADF encodings}
\label{sec:encadf}

%In this section we demonstrate, for purposes of providing ASP encodings for ADF reasoning, the use of a fact about the complexity of ASP programs.  The fact in question is that the combined complexity of ASP for programs with predicates of bounded arity  \cite{EiterFFW07}, just
%as the complexity of many of the reasoning tasks for ADFs, occupies the second and third level of the
%polynomial hierarchy (see Section \ref{sec:ASP:comp}). This allows for dynamic yet
%single shot encodings to fragments of ASP with matching complexity.

We construct ASP encodings $\pi_{\sigma}$ for the semantics $\sigma \in \{\adm,\com,\pef,\ground, \stb\}$ such that there is a %certain one to one
 correspondence between the $\sigma$ interpretations of an ADF $D = (S,C)$ and the answer sets of $\pi_{\sigma}(D)$ (the encoding function $\pi_{\sigma}$ applied to $D$).  More precisely, we will use atoms $\Pass{s,x}$ with
$s \in S, x \in \truthvals$  to represent ADF interpretations in our encodings.
%Troughout this section we represent the truth values true and false with $\tvft$ and $\tvtt$ rather than $\tvf$ and $\tvt$ for technical reasons: in the encodings we make use of ASP built in binary arithmetic functions and comparison predicates for evaluating propositional formulas such as ADF acceptance conditions.
An interpretation $\adfint$ of $D$ and a set of ground atoms (interpretation of an ASP program) $\aspint$ correspond to each other, $\adfint \cong \aspint$, whenever for every $s \in S$, $\adfint(s) = x$ iff  $\Pass{s,x} \in \aspint$. We overload $\cong$ to get the correspondence between
sets of interpretations and sets of answer sets
we aim for.

%\vspace{1cm}
\begin{defn}
\label{defcorr}
Given a set of (ADF) interpretations $\adfints$ and a collection of sets of ground atoms (ASP interpretations) $\aspints$,
we say that $\adfints$ and $\aspints$ correspond, $\adfints \cong \aspints$, if
\begin{enumerate}
\item for every $\adfint \in \adfints$ there is an $\aspint \in \aspints$ s.t. $\adfint \cong \aspint$;
\item for every $\aspint \in \aspints$ there is a  $\adfint \in \adfints$ s.t. $\adfint \cong \aspint$.
\end{enumerate}
\end{defn}

%\noindent The cardinalities of $\adfint$ and $\aspint$ such that $\adfint \cong \aspint$ may be different, but all that we are interested in in this section is that 
%Note that the definition of the correspondence $\adfint \cong \aspint$ for an ADF interpretation $\adfint$ and an ASP interpretation $\aspint$ means that items 1 and 2 in Definition \ref{defcorr} imply that $|\adfints| = |\aspints|$ \todo{The latter is not true... revise}.

Having encodings $\pi_{\sigma}$ for $\sigma \in \{\adm,\com,\pef,\ground, \stb\}$ for which $\sigma(D) \cong \AS(\pi_{\sigma}(D))$ for any ADF $D$ allows to enumerate the $\sigma$ interpretations of an ADF $D$ by reading the ADF interpretations that correspond (via $\cong$) to each $\aspint \in \AS(\pi_{\sigma}(D))$ off the predicates $\Pass{s,x} \in \aspint$  ($s \in S$, $x \in \truthvals$).  Results for credulous and skeptical reasoning for each of the semantics are obtained via the homonymous ASP reasoning tasks applied on our encodings.

\subsection{Encoding for the admissible semantics}
\label{sec::admadf}

In the course of presenting our dynamic ASP encodings for the admissible semantics we introduce several elements we will make use of throughout Section \ref{sec:encadf}.  Among these is that all encodings will assume a simple set of facts indicating the statements of the input ADF $D = (S, \{\ac{S}\}_{s \in S})$:

\begin{align*}
\pi_{\txt{arg}}(D) := \{ \Parg{s}. \mid s \in S \}. 
\end{align*}

\noindent Also, several of the encodings will need facts for encoding the possible truth values that can be assigned to a statement $s$ by a completion of an interpretation mapping $s$ to $\tvut$, $\tvtt$, and $\tvft$, respectively:

\begin{align*}
\pi_{\txt{lt}} :=
\{ \Pleq{\tvut,\tvft}. \ \Pleq{\tvut,\tvtt}. \ \Pleq{\tvtt,\tvtt}. \ \Pleq{\tvft,\tvft}. \}.
\end{align*}

\noindent Here, for instance the atoms $\Pleq{\tvut,\tvft}$ and $\Pleq{\tvut,\tvtt}$ together express that if an ADF interpretation maps a statement to the truth value $\tvut$ then a completion (of the interpretation in question) can map the same statement to the truth values $\tvft$ or $\tvtt$. Note in particular that $\Pleq{\tvut,\tvut} \not\in \pi_{\txt{lt}}$ since completions can map a statement only to the truth value $\tvft$ or $\tvtt$.

All of our encodings, including the one for the admissible semantics, follow the guess \& check methodology that is at the heart of the ASP paradigm \cite{JanhunenN16}. Here parts of a program delineate candidates for a solution to a problem.  These are often referred to as ``guesses''.  Other parts of the program, the ``constraints'', then check whether the guessed candidates are indeed solutions.  In the case of the encodings for ADFs the guessing part of the programs  outline possible assignments of truth values to the statements, i.e. an ADF interpretation.  For the three valued semantics, as the admissible semantics, the rules are as follows:

%{\small
\begin{align*}
\pi_{\txt{guess}}%
			:= \{ \Pass{S,\tvft} \derive & \Pnot\;\Pass{S,\tvtt}, \Pnot\;\Pass{S,\tvut}, \Parg{S}.\\
                            \Pass{S,\tvtt} \derive & \Pnot\;\Pass{S,\tvut}, \Pnot\;\Pass{S,\tvft}, \Parg{S}. \\
                            \Pass{S,\tvut} \derive & \Pnot\;\Pass{S,\tvft}, \Pnot\;\Pass{S,\tvtt}, \Parg{S}.\}.
\end{align*}
%}

We follow~\cite{BichlerMW2016} in encoding NP-checks in large non-ground rules having bodies with predicates of bounded arity.  In particular,  in all our encodings we will need rules encoding the semantic evaluation of propositional formulas; e.g. the evaluation of the acceptance conditions by completions of an interpretation. Given a propositional formula $\phi$, for this we introduce the function $\B$. For assignments of truth values ($\tvtt$ and $\tvft$) to the propositional variables in $\phi$, $\B(\phi)$ gives us a set of atoms corresponding to the propagation of the truth values to the subformulas of $\phi$ in accordance with the semantics of classical propositional logic.  The atoms make use of ASP variables $V_{\psi}$ where $\psi$ is a subformula of $\phi$.  The variables $V_{p}$, where $p$ is a propositional variable occurring in $\phi$, can be used by other parts of ASP rules employing the atoms in $\B(\phi)$ for purposes of assigning intended truth values to the propositional variables in $\phi$.  

For the definition of the atoms $\B(\phi)$ we rely on the ASP built in arithmetic functions $\texttt{\&}$ (bitwise AND), $\texttt{?}$ (bitwise OR), and $\texttt{-}$ (subtraction). We also use the built in comparison predicate $=$.  Let $\phi$ be a propositional formula over a set of propositional variables $P$; then the set of atoms in question is defined as

%{\small
\[
\B(\phi) \hspace{-1.1pt}:=\hspace{-1mm}
    \left\{
        \begin{array}{ll}
                \hspace{-2mm}
		\B(\phi_1) \cup 	
		\B(\phi_2) \cup 	
		\{ V_\phi = V_{\phi_1} \texttt{\&} V_{\phi_2} \} &
			\hspace{-2mm}\mbox{if\ } \phi=\phi_1\wedge\phi_2\\
                \hspace{-2mm}
		\B(\phi_1) \cup 	
		\B(\phi_2) \cup 	
		\{ V_\phi = V_{\phi_1} \texttt{?} V_{\phi_2} \} &
			\hspace{-2mm}\mbox{if\ } \phi=\phi_1\vee\phi_2\\
                \hspace{-2mm}
		\B(\psi) \cup 	
		\{ V_\phi = 1 \texttt{-} V_{\psi} \} &
			\hspace{-2mm}\mbox{if\ } \phi=\neg \psi	\\
                \hspace{-2mm}
		\emptyset &
			\hspace{-2mm}\mbox{if\ } \phi= p\in P
        \end{array}
	\right.
\]
%}

\noindent
where $V_\phi$,
$V_{\phi_1}$
$V_{\phi_2}$ and
$V_{\psi}$ are variables representing the subformulas of $\phi$.

Our encoding for the admissible semantics, $\pi_{\adm}$, is based on the fact that an interpretation $\adfint$ for an ADF $D$ is admissible iff for every $s \in S$ it is the case that

\begin{itemize}
\item if $v(s) = \tvtt$ then there is no $w \in \cpl{v}$ s.t. $w(\ac{s}) = \tvft$,
\item if $v(s) = \tvft$ then there is no $w \in \cpl{v}$ s.t. $w(\ac{s}) = \tvtt$.
\end{itemize}

\noindent This is a simple consequence of the definition of the admissible semantics. Any $w \in \cpl{v}$ which contradicts this simple observation (e.g. $v(s) = \tvtt$ and $w(\ac{s}) = \tvft$) is a ``counter-model'' to $\adfint$ being an admissible interpretation.  The constraining part of our encoding for the admissible semantics essentially disallows guessed assignments of truth values to the statements of an ADF corresponding to ADF interpretations which have counter-models to them being admissible.

To encode the constraints of our encoding we need auxiliary rules firing when the guessed assignments have counter-models to them being admissible.  These rules, two for each $s \in S$, make use of bodies $\Bl_s$ where $\B(\ac{s})$ is employed to evaluate the acceptance conditions by the completions.  The latter are obtained by setting variables $V_t$ for $t \in \prts{s}$ with the adequate truth values by using the predicates $\PassP$ and $\PleqP$ defined in $\pi_{\txt{guess}}$ and $\pi_{\txt{lt}}$:

\[
\Bl_s := \{ \Pass{t,Y_t}, \Pleq{Y_t,V_t} \mid
t\in\prts{s}\} \cup \B(\ac{s}).
\]

The two rules for every statement $s \in S$
have heads $\Psat{s}$ and $\Punsat{s}$
that %
fire
in case there is some completion of the interpretation corresponding to the assignments
guessed in the program fragment $\pi_{\text{guess}}$
such that the acceptance condition $\ac{s}$ evaluates to $\tvtt$ and $\tvft$, respectively:

%{\small
\begin{align*}
\pi_{\txt{sat}}(D) := \{
& \Psat{s} \derive
\Bl_s, V_{\ac{s}} = 1. \\
& \Punsat{s} \derive
\Bl_s, V_{\ac{s}} = 0.
\mid
s\in S
\}.
\end{align*}
%}

\noindent Here for an ADF interpretation $v$ ``guessed'' via the fragment $\pi_{\txt{guess}}$ (and encoded using the predicate ``$\PassP$'') and a $s \in S$, $\Psat{s}$ is derived whenever $v(s) = \tvf$ and there is a $w \in \ext{v}$ for which $w(\phi_s) = \tvt$.  On the other hand, $\Punsat{s}$ is derived whenever $v(s) = \tvt$ and there is a $w \in \ext{v}$ for which $w(\phi_s) = \tvf$. The atoms using predicates ``$\PassP$'' and ``$\PleqP$'' in $\Bl_s$ are used to encode possible assignments a completion $w \in \ext{v}$ can take, while the atoms in $\B(\ac{s})$ propagate such assignments to $\ac{s}$ in accordance with the semantics of classical propositional logic by making use of ASP built-ins and auxiliary variables $V_{\psi}$ for every subformula $\psi$ of $\ac{s}$.  %Note in particular that while assignments corresponding to guessed interpretations are encoded making use of the ``$\PassP$'' predicate and assignments corresponding to completions are encoded making use of the ``$\PleqP$'' predicate, the semantic evaluation of the acceptance conditions is encoded making use of ASP built-in predicates for equality and basic arithmetic.         

The encoding for the admissible semantics %
 now results from compounding the program fragments $\pi_{\txt{arg}}(D)$, $\pi_{\txt{lt}}$, $\pi_{\txt{guess}}$, and $\pi_{\txt{sat}}(D)$
together with ASP constraints which filter out guessed assignments of statements to truth values (via $\pi_{\txt{guess}}$) corresponding to interpretations of $D$ having counter-models to being admissible:

%{\small
\begin{align*}
&\pi_{\adm}(D) := %
\pi_{\txt{arg}}(D) \ \cup \
\pi_{\txt{lt}} \ \cup \
\pi_{\txt{guess}} \ \cup \
\pi_{\txt{sat}}(D) \ \cup \\
&\{  \derive \Parg{S}, \Pass{S,\tvtt}, \Punsat{S}.
\;\,\derive \Parg{S}, \Pass{S,\tvft}, \Psat{S}. \}.
\end{align*}
%}

\noindent For instance the last constraint in $\pi_{\adm}(D)$ disallows guessed interpretations $v$ for which there is a $s \in S$ and $w \in \ext{v}$ such that $v(s) = \tvft$ and $w(\ac{s}) = \tvtt$.  

Proposition \ref{thm:adm} formally states that $\pi_{\adm}$ is an adequate encoding function.  For the proof, which is prototypical for most of the proofs of correctness in this work, we use the notation

\begin{align*}
I_p := \{p(t_1,\ldots,t_n) \in \aspint\}.
\end{align*}

\noindent For an ASP interpretation $\aspint$ (set of ground atoms), $I_p$ represents $\aspint$ projected onto the predicate $p$ (with arity $n$).  

\begin{propn}
\label{thm:adm}
For every ADF $D$ it holds that $\adm(D) \cong \AS(\pi_{\adm}(D))$.
\end{propn}

\begin{proof} Let $D = (S, \{\ac{S}\}_{s \in S})$ be an ADF and $\adfint \in \adm(D)$.  Let also

\begin{align*}
\aspint :=& \{ \Parg{s} \mid s \in S \} \ \cup \\
& \{ \Pleq{\tvut,\tvft}, \Pleq{\tvut,\tvtt}, \Pleq{\tvtt,\tvtt}, \Pleq{\tvft,\tvft} \} \ \cup \\
& \{ \Pass{s,x} \mid s \in S, \adfint(s) = x\} \ \cup \\
& \{ \Psat{s} \mid \txt{if there is a } w \in \ext{\adfint} \txt{ s.t. } w(\phi_s) = \tvtt\} \ \cup \\
& \{ \Punsat{s} \mid \txt{if there is a } w \in \ext{\adfint} \txt{ s.t. } w(\phi_s) = \tvft\} \\
\end{align*}

\noindent be a set of ground atoms (such that $\adfint \cong \aspint$).  We prove now that $\aspint \in \AS(\pi_{\adm}(D))$.

We start by proving that $\aspint$ satisfies $\ared{\pi_{\adm}(D)}{\aspint}$.  First note that $\aspint$ satisfies $\ared{\pi_{\txt{arg}}(D)}{\aspint} = \pi_{\txt{arg}}(D)$ as well as $\ared{\pi_{\txt{lt}}}{\aspint} = \pi_{\txt{lt}}$ since all the atoms making up the facts in these two modules are in $\aspint$ (first two lines of the definition of $\aspint$).  $\aspint$ also satisfies    

\begin{align*}
\ared{\pi_{\txt{guess}}}{\aspint} =& \{ \Pass{s,x} \leftarr \Parg{s}. \mid s \in S, \\
& \ \ \Pass{s,y} \not\in \aspint,
\Pass{s,z} \not\in \aspint, x \in \truthvals, y,z \in (\truthvals \setminus \{x\})\}
\end{align*}

\noindent since, first of all, $\Parg{s} \in \aspint$ iff $s \in S$ by the first line of the definition of $\aspint$ (and the fact that the predicate $\PargP$ does not appear in the head of any rules other than $\ared{\pi_{\txt{arg}}(D)}{\aspint} = \pi_{\txt{arg}}(D)$).  Secondly, for any $s \in S$, $\Pass{s,x} \in \aspint$ whenever $\Pass{s,y} \not\in \aspint$ and $\Pass{s,z} \not\in \aspint$ for $x \in \truthvals$ and $y,z \in (\truthvals \setminus \{x\})$ by the fact that $\adfint \cong \aspint$ (third line of the definition of $\aspint$).

Now consider the rule $r \in \pi_{\txt{sat}}(D)$ with $H(r) = \Psat{S}$ and a substitution $\theta$ s.t.  $\theta r \in \ared{\pi_{\txt{sat}}(D)}{\aspint}$.  %  
	This means that $\theta r$ is of the form 

\begin{align*}
\Psat{s} \derive \theta \Bl_s, \theta (V_{\ac{s}} = 1).
\end{align*}

\noindent with

\begin{align*}
\theta \Bl_s = \{ \Pass{t,y_t}, \Pleq{y_t,v_t} \mid t\in\prts{s} \} \cup \theta \B(\ac{s})
\end{align*}

\noindent and where $\theta(Y_t) = y_t$, $\theta(V_t) = v_t$.  If $B(\theta r) \in \aspint$, it must be the case that $y_t \in \truthvals$, $v_t \in \{\tvft,\tvtt\}$ for $t \in \prts{s}$ and $\theta \B(\ac{s}) \in \aspint$.  Now it should be easy for the reader to see that from the fact that  $\{ \Pass{t,y_t}, \Pleq{y_t,v_t} \mid t\in\prts{s} \} \subseteq \aspint$ and $\adfint \cong \aspint$ it is the case that $w \in \ext{v}$ for the ADF interpretation $w$ defined as $w(t) = v_t$ for every $t \in \prts{s}$.  It is also simple to establish that  $\theta \B(\ac{s}) \in \aspint$ and $\theta(V_{\ac{s}} = 1) \in I$ imply that $w(\ac{s}) = \tvtt$.  Hence, by the fourth line of the definition of $\aspint$ we have $\Psat{s} \in \aspint$, i.e. $\aspint$ satsifies $\theta r$.  In the same manner, by the fifth line of the definition of $\aspint$ it follows that $\aspint$ satisfies any grounding $\theta r \in \ared{\pi_{\txt{sat}}(D)}{\aspint}$ for the rule $r$ s.t. $H(r) = \Punsat{S}$.  In conclusion, $\aspint$ satisfies $\ared{\pi_{\txt{sat}}(D)}{\aspint}$.

Let us turn now to a ground instance $r \in \ared{\pi_{\adm}(D)}{\aspint}$

\begin{align*}
\leftarr \Parg{s},\Pass{s,\tvft},\Psat{s}.
\end{align*}

\noindent of the constraint

\begin{align*}
\leftarr \Parg{S},\Pass{S,\tvft},\Psat{S}. 
\end{align*}

\noindent $\in \pi_{\adm}(D)$. By the fourth line of the definition of $\aspint$,  $\Psat{s} \in \aspint$ iff there is a $w \in \ext{v}$ s.t. $w(\ac{s}) = \tvtt$.   But then by the fact that $v \in \adm(D)$ and $v \cong \aspint$, $\Pass{s,\tvft} \not\in \aspint$, i.e. $r$ can not be satisfied by $\aspint$.  In the same manner also any ground instance in $\ared{\pi_{\adm}(D)}{\aspint}$ of the constraint 

\begin{align*}
\leftarr \Parg{S},\Pass{S,\tvtt},\Punsat{S}.
\end{align*}

\noindent can not be satisfied by $\aspint$.

We have established that $\aspint$ satisfies $\ared{\pi_{\adm}(D)}{\aspint}$.  We continue our proof of $\aspint \in \AS(\pi_{\adm}(D))$ by now showing that there is no $\aspint' \subset \aspint$ that satisfies $\ared{\pi_{\adm}(D)}{\aspint}$.

In effect, consider any other $\aspint'$ that satisfies $\ared{\pi_{\adm}(D)}{\aspint}$.  Note first of all that then $\aspint'_{\PargP} \supseteq \aspint_{\PargP}$ and $\aspint'_{\PleqP} \supseteq \aspint_{\PleqP}$ because both $\aspint'$ and $\aspint$ satisfy $\ared{\pi_{\txt{arg}}(D)}{\aspint}$ as well as $\ared{\pi_{\txt{lt}}}{\aspint}$.  Hence also $\aspint'_{\PassP} \supseteq \aspint_{\PassP}$ because $\aspint'$ satisfies  $\ared{\pi_{\txt{guess}}}{\aspint}$ (see the proof of $\aspint$ satisfies $\ared{\pi_{\adm}(D)}{\aspint}$ for the structure of  $\ared{\pi_{\txt{guess}}}{\aspint}$) and $\aspint'_{\PargP}\supseteq \aspint_{\PargP}$, i.e. $B(r) \subseteq \aspint'$ for every $r \in \ared{\pi_{\txt{guess}}}{\aspint}$.  But then, since   $\aspint'_{\PargP} \supseteq \aspint_{\PargP}$ and  $\aspint'_{\PassP} \supseteq \aspint_{\PassP}$, and $\aspint'$ satisfies all the comparison predicates with arithmetic functions that $\aspint$ does by definition, $\aspint'$ satisfies all the rules in $\ared{\pi_{\txt{sat}}(D)}{\aspint}$ that $\aspint$ does (see again the proof of $\aspint$ satisfies $\ared{\pi_{\adm}(D)}{\aspint}$ for the form of such rules).  Hence, also $\aspint'_{\PsatP} \supseteq \aspint_{\PsatP}$ and $\aspint'_{\PunsatP} \supseteq \aspint_{\PunsatP}$.  In conclusion, $\aspint' \supseteq \aspint$.

Since $\aspint'$ was general we derive that there is no $\aspint' \subset \aspint$ that satisfies $\ared{\pi_{\adm}(D)}{\aspint}$.  Together with the fact that $\aspint$ satisfies $\ared{\pi_{\adm}(D)}{\aspint}$ we have that $\aspint \in \AS(\pi_{\adm}(D))$.

We now turn to proving that for any $\aspint \in \AS(\pi_{\adm}(D))$ it holds that $\adfint \in \adm(D)$ for $\adfint \cong \aspint$.  Note first that for such an $\aspint$, since $\aspint$ satisfies $\ared{\pi_{\txt{arg}}(D)}{\aspint} = \pi_{\txt{arg}}(D)$
%\begin{align*}
%\{\Parg{s}. \mid s \in S\} \subseteq \ared{\pi_0(D)}{\aspint} = \pi_0(D)
%\end{align*}
%\noindent
as well as

\begin{align*}
\ared{\pi_{\txt{guess}}}{\aspint} =& \{ \Pass{s,x} \leftarr \Parg{s}. \mid s \in S, \\
& \ \ \Pass{s,y} \not\in \aspint,
\Pass{s,z} \not\in \aspint, x \in \truthvals, y,z \in (\truthvals \setminus \{x\})\},
\end{align*}

\noindent for every $s \in S$ there is a $x \in \truthvals$ such that $\Pass{s,x} \in \aspint$.  Also, $\Pass{s,x} \in \aspint$ whenever $\Pass{s,y} \not \in \aspint$ and $\Pass{s,z} \not \in \aspint$ for $y,z \in (\truthvals \setminus \{x\})$.  I.e. $v$ s.t. $v \cong \aspint$ is well defined.

Now assume that $\adfint \not \in \adm(D)$.  Then there are $s \in S$, $w \in \ext{\adfint}$ for which either i) $\adfint(s) = \tvtt$ and $w(\ac{s}) = \tvft$ or  ii) $\adfint(s) = \tvft$ and $w(\ac{s}) = \tvtt$.  Let us consider the case i).  In that case consider a substitution $\theta$ for the rule $r \in \pi_{\txt{sat}}(D)$

\begin{align*}
\Punsat{s} \derive \B_s, V_{\ac{s}} = 0.
\end{align*}

\noindent where

\[
\B_s = \{ \Pass{t,Y_t}, \Pleq{Y_t,V_t} \mid
t\in\prts{s}\} \cup \B(\ac{s}).
\]

\noindent The substitution $\theta$ is defined as $\theta(Y_t) = v(t)$ and $\theta(V_t) = w(t)$ for every $t \in \prts{s}$.  Since $\adfint \cong \aspint$ we have that $\Pass{t,\theta(Y_t)} \in \aspint$ for every $t \in \prts{s}$.  Also $\Pleq{\theta(Y_t),\theta(V_t)} \in \aspint$ since $\aspint$ satisfies $\ared{\pi_{\txt{lt}}}{\aspint}$.

%\begin{align*}
%\{ \Pleq{\tvut,\tvft}. \Pleq{\tvut,\tvtt}. \Pleq{\tvtt,\tvtt}. \Pleq{\tvft,\tvft}. \} \subseteq \ared{\pi_0(D)}{\aspint} = \pi_0(D).
%\end{align*}

\noindent Now, by definition $\theta \B(\ac{s}) \subseteq \aspint$ and from $w(\ac{s}) = \tvft$ it is easy to see that it follows that also $\theta(V_{\ac{s}} = 0) \in \aspint$, i.e. $\theta r \in \ared{\pi_{\txt{sat}}(D)}{\aspint}$ and $B(\theta r) \subseteq \aspint$.  This means that also $\Punsat{s} \in \aspint$.  As a consequence we have that $B(r') \subseteq \aspint$ for the constraint $r'$

\begin{align*}
\leftarr \Parg{s}, \Pass{s, \tvtt}, \Punsat{s}.
\end{align*}

\noindent in $\ared{\pi_{\adm}(D)}{\aspint}$.  This is a contradiction to $\aspint \in \AS(\pi_{\adm}(D))$. From the case  ii) $\adfint(s) = \tvft$ and $w(\ac{s}) = \tvtt$ a contradiction can be derived in analogous manner.  Hence, $v \in \adm(D)$ must be the case.     
\end{proof}

\begin{example}
\label{ex:adm}
Considering the ADF $D$ from Example~\ref{ex:adf},
%(recall $\ac{a} {=} a \vee \neg b, \ac{b} {=} \neg a$),
%
$\pi_{\adm}(D)$ (as implemented by our system \PNAME{YADF} with minor formatting for purposes of readability; see Section \ref{sec:exper}) looks as follows:

\begin{verbatim}
arg(a).
arg(b).
arg(c).
leq(u,0).
leq(u,1).
leq(0,0).
leq(1,1).
asg(S,u) :- arg(S),not asg(S,0),not asg(S,1).
asg(S,0) :- arg(S),not asg(S,1),not asg(S,u).
asg(S,1) :- arg(S),not asg(S,u),not asg(S,0).
sat(a) :- asg(b,Y0),leq(Y0,V0),V1=1-V0,V2=V1?V0,V2=1.
sat(b) :- asg(b,Y0),leq(Y0,V0),V0=1.
sat(c) :- asg(c,Y0),leq(Y0,V0),asg(b,Y1),leq(Y1,V1),
                                  V3=1,V3=V2?V1,V2=1-V0.
inv(a) :- asg(b,Y0),leq(Y0,V0),V1=1-V0,V2=V1?V0,V2=0.
inv(b) :- asg(b,Y0),leq(Y0,V0),V0=0.
inv(c) :- asg(c,Y0),leq(Y0,V0),asg(b,Y1),leq(Y1,V1),
                                  V3=V2?V1,V3=0,V2=1-V0.
:- arg(S),asg(S,1),inv(S).
:- arg(S),asg(S,0),sat(S).
\end{verbatim}

\noindent
A possible output of an ASP solver (the current one is the simplified output of \PNAME{clingo} version 4.5.4) given this instance looks as follows (only showing $\PassP$, $\PsatP$, and $\PunsatP$ predicates):

\begin{verbatim}
Answer: 1
asg(c,u) asg(b,0) asg(a,u) sat(c) sat(a) inv(c) inv(b)
Answer: 2
asg(c,u) asg(b,0) asg(a,1) sat(c) sat(a) inv(c) inv(b)
Answer: 3
asg(c,u) asg(b,1) asg(a,u) sat(b) sat(c) sat(a)
Answer: 4
asg(c,u) asg(b,1) asg(a,1) sat(b) sat(c) sat(a)
Answer: 5
asg(c,u) asg(b,u) asg(a,1) sat(b) sat(c) sat(a) inv(c) inv(b)
Answer: 6
asg(c,u) asg(b,u) asg(a,u) sat(b) sat(c) sat(a) inv(c) inv(b)
Answer: 7
asg(c,1) asg(b,1) asg(a,u) sat(b) sat(c) sat(a)
Answer: 8
asg(c,1) asg(b,1) asg(a,1) sat(b) sat(c) sat(a)
SATISFIABLE
\end{verbatim}
\hfill$\diamond$
\end{example}

The encoding $\pi_{\adm}$ allows to enumerate the admissible interpretations of an ADF $D$ from the answer sets of $\pi_{\adm}(D)$ (as explained in the opening paragraphs of Section~\ref{sec:encadf}).  Skeptical reasoning for the admissible semantics is trivial (as the interpretation mapping every statement to $\tvut$ is always admissible), but note that via credulous reasoning for ASP programs we directly obtain results for credulous reasoning \wrt the admissible semantics from $\pi_{\adm}(D)$ (for any ADF $D$).  The latter translation and thus the encoding $\pi_{\adm}$ is adequate from the point of view of the complexity (see Table \ref{tbl:complResGrappa}) as $\pi_{\adm}(D)$ is a normal logic program for any ADF $D$.  Also, given our recursive definition of the evaluation of the acceptance conditions within ASP rules, the arity of predicates in our encodings are bounded (in fact, the maximum arity of predicates is two).

\subsection{Encoding for the complete semantics}

For the ASP encoding of the complete semantics we only need to add two constraints
to the encoding of the admissible semantics.
These express a further condition that an interpretation $\adfint$ for an ADF $D = (S, \{\ac{S}\}_{s \in S})$ has to fulfill to be complete, in addition to not having counter-models for being an admissible interpretation as expressed in Section \ref{sec::admadf}.  The condition in question is that for every $s \in S$:

\begin{itemize}
\item if $v(s) = \tvu$ then there are $w_1,w_2 \in \cpl{v}$ s.t. $w_1(\ac{s}) = \tvft$ and $w_2(\ac{s}) = \tvtt$.
\end{itemize}

\noindent Expressing this condition in the form of constraints gives us the encoding

\begin{align*}
\pi_{\com}(D) := &\pi_{\adm}(D) \ \cup \ \\ \{&\derive \Parg{S}, \Pass{S,\tvu}, \Pnot\;\Punsat{S}. \\
&\derive \Parg{S}, \Pass{S,\tvu}, \Pnot\;\Psat{S}. \}.
\end{align*}

\begin{propn}\label{thm:com} For every ADF $D$ it holds that $\com(D) \cong \AS(\pi_{\com}(D))$.
\end{propn}

\begin{proof}(sketch) The proof extends that of Proposition \ref{thm:adm}.  Only the additional constraints used in the encoding for the complete semantics (w.r.t. the encoding for the admissible semantics) need to be accounted for.  
\end{proof}

%\todo{MD@MD: Update depending on what we decide on complexity adequacy}

The encoding $\pi_{\com}$ allows to enumerate the complete interpretations of an ADF $D$ by applying the encoding on $D$ ($\pi_{\com}(D)$) and considering $\AS(\pi_{\com}(D))$.  
%is of the enumeration problem w.r.t. the complete semantics.  %As for the admissible semantics, encodings for the reasoning problems w.r.t. the complete semantics can be obtained via the corresponding ASP reasoning problems applied on $\pi_{\com}(D)$. In particular,
%com
Since credulous acceptance for the complete semantics is equivalent to credulous acceptance for the admissible semantics, we obtain a complexity adequate means of computing credulous acceptance for the complete semantics via applying credulous (ASP) reasoning on $\pi_{\adm}(D)$ ($\pi_{\adm}$ being the encoding presented in Section \ref{sec::admadf}).  Applying credulous reasoning on $\pi_{\com}(D)$ is nevertheless also an option\footnote{Which may in fact (because of redundancy) be more efficient in practice.}.

%obtain complexity adequate encodings for the complete semantics via the encoding presented in Section \ref{sec::admadf}.  Applying credulous ASP reasoning on $\pi_{\com}$ is nevertheless also an option\footnote{Which may in fact (because of redundancy) be more efficient in practice.}.  

%Skeptical acceptance, on the other hand, is equivalent to skeptical (= credulous) acceptance for the grounded semantics; we therefore obtain encodings via the encoding for the grounded semantics presented in sections \ref{sec:yadf:ground} and \ref{sec:yadf:diamond}.

\subsection{Saturation encoding for the preferred semantics}

For the encoding of the preferred semantics we make use of the saturation technique \cite{DBLP:journals/amai/EiterG95}; see  %
\cite{CharwatDGWW15} for its use in computing the preferred extensions of Dung AFs.
The saturation technique allows checking that a property holds
for a \emph{set} of guesses within a disjunctive ASP program,
by generating a unique
``saturated'' guess that ``verifies'' the property for any such guess.
Existence of a non-saturated guess hence implies that the property of
interest does not hold for the guess in question.  
%This implies that whenever there is some guess for which the property of interest
%does not hold,

In the encoding of the preferred semantics for an ADF $D$
we extend $\pi_{\adm}(D)$ by
making use of
the saturation technique
to verify that all interpretations of $D$
that are greater w.r.t.\ $\leq_i$ than the interpretation
determined by the assignments guessed in the program fragment $\pi_{\txt{guess}}$
are either identical to the interpretation in question or not admissible.
As a consequence, the relevant interpretation must be preferred according to the definition of this semantics for ADFs.

The module $\pi_{\text{guess2}}$ amounts to ``making a second guess'' (indicated by the predicate $\PassTP$)
extending the ``first guess'' ($\PassP$)
from
$\pi_{\text{guess}}$:

%{\small
\begin{align*}
\pi_{\text{guess2}} :=  \{ &\PassT{S,\tvft} \derive \Pass{S,\tvft}.\\
&\PassT{S,\tvtt} \derive \Pass{S,\tvtt}. \\
& \PassT{S,\tvtt} \vee \PassT{S,\tvft} \vee \PassT{S,\tvut} \derive \Pass{S,\tvut}.
\}.
\end{align*}
%}
%
%
%

\noindent Note that the first two rules express that if an ADF interpretation $v$ corresponding to the ``first guess'' (captured via the predicate $\PassP$) maps a statement to either $\tvf$ or $\tvt$ then so does a interpretation $v'$ corresponding to the ``second guess'' (captured via the predicate $\PassTP$).  The last rule, on the other hand, indicates that if the first guess maps a statement to $\tvut$ then the second guess can map the statement to either of the truth values $\tvut$, $\tvf$, or $\tvt$.  Thus $v' \geq_i v$ is guaranteed.  

The fragment $\pi_{\txt{sat2}}(D)$
will allow us
to check whether the second guess
obtained from $\pi_{\text{guess2}}$
is admissible:
%

%{\small
\begin{align*}
\pi_{\txt{sat2}}(D) := \{
& \PsatT{s} \derive \Bl2_s, V_{\ac{s}} = 1. \\
& \PunsatT{s} \derive \Bl2_s, V_{\ac{s}} = 0.
\mid
s\in S \}
\end{align*}

\noindent with

\begin{align*}
\Bl2_s  := \{ & \PassT{t,Y_t}, \Pleq{Y_t,V_t} \mid t\in\prts{s}\} \cup \B(\ac{s}).
\end{align*}

\noindent
The only difference between the fragment $\pi_{\txt{sat2}}(D)$ and the fragment $\pi_{\txt{sat}}(D)$ that we introduced in Section~\ref{sec::admadf}
is that we now evaluate acceptance conditions \wrt completions of the second guess given via the predicate
$\PassTP$.

The following program fragment guarantees that the atom $\Psaturate$
is derived whenever the second guess (computed via $\pi_{\text{guess2}}$)
is either identical
(first rule of $\pi_{\txt{check}}(D)$)
to the first guess (computed via the module $\pi_{\text{guess}}$)
or is not admissible (last two rules of $\pi_{\txt{check}}(D)$).
We will say that in this case the second guess is \emph{not} a counter-example
to the first guess corresponding to a preferred interpretation of $D$.
We here assume that the statements $S$ of $D$ are numbered, i.e. $S = \{s_1,\ldots,s_k\}$ for a $k \geq 1$:

%\begin{small}
\begin{align*}
\pi_{\txt{check}}(D) := \{
 \Psaturate \derive & \Pass{s_1,X_1},\PassT{s_1,X_1},
\dots\\
& \Pass{s_k,X_k},\PassT{s_k,X_k}. \\
 \Psaturate \derive & \PassT{S,\tvtt}, \PunsatT{S}. \\
 \Psaturate \derive & \PassT{S,\tvft}, \PsatT{S}. \}.
\end{align*}
%\end{small}

%OLD VERSION WITH REDUNDANT PREDICATES \Parg{S}:
%\begin{small}
%\begin{flalign*}
%\pi_{\txt{check}}(D) := \{
% \Psaturate \derive & \Pass{s_1,X_1},\PassT{s_1,X_1},
%
%
%\dots\\
%& \Pass{s_k,X_k},\PassT{s_k,X_k}. \\
% \Psaturate \derive & \Parg{S}, \PassT{S,\tvt}, \PunsatT{S}. \\
% \Psaturate \derive & \Parg{S}, \PassT{S,\tvf}, \PsatT{S}. \}
%
%
%
%\end{flalign*}
%\end{small}

%
The module $\pi_{\txt{saturate}}$ now assures that whenever the atom $\Psaturate$ is derived,
first of all $\PassT{s,\tvft}$, $\PassT{s,\tvtt}$, and $\PassT{s,\tvut}$ are derived
for every $s \in S$ for which $\Pass{s,\tvut}$ has been derived.  Also, $\PsatT{s}$ and $\PunsatT{s}$
are derived for every $s \in S$:
%

%\begin{small}
\begin{align*}
\pi_{\txt{saturate}} := \{
 & \PassT{S,\tvft} \derive \Pass{S,\tvut}, \Psaturate. \\
 & \PassT{S,\tvtt} \derive \Pass{S,\tvut}, \Psaturate. \\
 & \PassT{S,\tvut} \derive \Pass{S,\tvut}, \Psaturate. \\
 & \PsatT{S} \derive \Parg{S}, \Psaturate.\; \\
 & \PunsatT{S} \derive \Parg{S}, \Psaturate. \}.
\end{align*}
%\end{small}

%OLD VERSION, WITH REDUNDANT PREDICATES \Parg{S}
%\begin{small}
%\begin{flalign*}
%\pi_{\txt{saturate}} := \{
% & \PassT{S,\tvf} \derive \Parg{S}, \Pass{S,\tvu}, \Psaturate. \\
% & \hspace{-5mm}\PassT{S,\tvt} \derive \Parg{S}, \Pass{S,\tvu}, \Psaturate. \\
% & \hspace{-5mm}\PassT{S,\tvu} \derive \Parg{S}, \Pass{S,\tvu}, \Psaturate. \\
%
% & \hspace{-5mm}\PsatT{S} \derive \Parg{S}, \Psaturate.\;
% \PunsatT{S} \derive \Parg{S}, \Psaturate. \}
%\end{flalign*}
%\end{small}

\noindent The effect of this fragment is that whenever all the ``second guesses'' (computed via $\pi_{\text{guess2}}$)
are \emph{not} counter-examples to the first guess
(computed via $\pi_{\text{guess}}$)
corresponding to a preferred interpretation of $D$,
then all the answer sets will be saturated on the predicates $\txt{asg2}$, $\txt{sat2}$, and $\txt{inv2}$,
i.e.\ the same ground instances of these predicates will be included in any answer set.  Thus, all answer sets (corresponding to the ADF interpretation determined by the first guess) will be indistinguishable on the new predicates used for the encoding of the preferred interpretation; meaning: those not in $\pi_{\adm}(D)$.  On the other hand, were there to be a counter-example to the first guess corresponding to a preferred interpretation of $D$, then a non-saturated and hence smaller (w.r.t $\subseteq$) answer set could be derived.  We disallow the latter by adding to the program fragments $\pi_{\adm}(D)$, $\pi_{\txt{guess2}}$,  $\pi_{\txt{sat2}}(D)$, $\pi_{\txt{check}}(D)$, $\pi_{\txt{saturate}}$, a constraint filtering out precisely such answer sets.  The latter being those for which the atom $\PsaturateP$ is \emph{not} derived.   
%
%
%Finally, all that needs to be added to the program fragments $\pi_{0}(D)$, $\pi_{\txt{guess}}(D)$, $\pi_{\txt{sat}}(D)$,  for the encoding of the preferred semantics, $\pi_{\txt{\pef}}(D)$, is a constraint disallowing those assignments which have a counter-example to them corresponding to a preferred interpretation of
%
%$D$, i.e. those for which the atom $\txt{saturate}$ is \emph{not} derived.
We thus arrive at the following encoding for the preferred semantics:

\begin{align*}
\pi_{\pef}(D): = &\pi_{\adm}(D) \ \cup \
\pi_{\txt{guess2}} \ \cup \ 
\pi_{\txt{sat2}}(D) \ \cup \ \\
&\pi_{\txt{check}}(D) \ \cup \
\pi_{\txt{saturate}} \ \cup \
\{ \txt{\derive \Pnot\;\Psaturate.} \}.
\end{align*}

\begin{propn}
\label{thm:prf}
For every ADF $D$ it holds that $\pef(D) \cong \AS(\pi_{\pef}(D))$.
%where $\pi_{\pef}(D)=\pi_{\adm}(D)\cup
%\pi_{\txt{guess2}} \cup
%\pi_{\txt{sat2}}(D) \cup
%\pi_{\txt{check}}(D) \cup
%\pi_{\txt{saturate}} \cup
%\{ \txt{\derive \Pnot\;\Psaturate.} \}$.
%
%
%
%
\end{propn}

\begin{proof} Let $D = (S, \{\ac{S}\}_{s \in S})$ be an ADF and $\adfint \in \pef(D)$.  Let also

\begin{align*}
\aspint' :=& \aspint \cup \{\PassT{s,\tvtt} \mid s \in S, \adfint(s) = \tvtt\} \cup \{\PassT{s,\tvft} \mid s \in S, \adfint(s) = \tvft\}  \cup \satint{\aspint}
\end{align*}

\noindent be a set of ground atoms where $\aspint$ is defined as in the ``only if'' direction of the proof of Proposition \ref{thm:adm} (hence, $\adfint \cong \aspint'$).  Moreover, $\satint{\aspint}$ is the set of ground atoms forming the ``saturation'' of the predicates $\PassTP$, $\PsatTP$, $\PunsatTP$, $\PsaturateP$ ($\PassTP$ is saturated only for $s \in S$ s.t. $v(s) = \tvut$) defined as

\begin{align*}
\satint{\aspint} :=&  \{\PassT{s,x} \mid s \in S, x \in \{\tvtt,\tvft,\tvut\}, \adfint(s) = \tvut\} \ \cup \\
& \{ \PsatT{s} \mid s \in S \} \ \cup \\
& \{ \PunsatT{s} \mid s \in S \} \ \cup \\
& \{ \Psaturate \}.
\end{align*}

Note first that since none of the predicates occurring in  $\aspint' \setminus \aspint$ appear in $\pi_{\adm}(D)$, we have that $\ared{\pi_{\adm}(D)}{\aspint'} = \ared{\pi_{\adm}(D)}{\aspint}$.  As thus also all of the atoms appearing in $\ared{\pi_{\adm}(D)}{\aspint'}$ that are in $\aspint'$ are those which are in $\aspint$, we have that $\aspint'$ and $\aspint$ satisfy the bodies and heads of the same rules in $\ared{\pi_{\adm}(D)}{\aspint'}$.  By the proof of the ``only if'' direction of Proposition \ref{thm:adm} (i.e. that $\aspint$ satisfies $\ared{\pi_{\adm}(D)}{\aspint} = \ared{\pi_{\adm}(D)}{\aspint'}$)  it then follows that $\aspint'$ satisfies $\ared{\pi_{\adm}(D)}{\aspint'}$.

$\aspint'$ also satisfies each of $\ared{\pi_{\txt{sat2}}(D)}{\aspint'} = \agrd{\pi_{\txt{sat2}}(D)}$, $\ared{\pi_{\txt{check}}(D)}{\aspint'} = \agrd{\pi_{\txt{check}}(D)}$  as the heads of all possible ground instances of the rules of each of the modules $\pi_{\txt{sat2}}(D)$ and $\pi_{\txt{check}}(D)$  are contained in $\satint{\aspint} \subset \aspint'$.   Moreover, $\aspint'$ satisfies all groundings of the first two rules of $\pi_{\txt{guess2}}$ (that are in $\ared{\pi_{\txt{guess2}}}{\aspint'} = \agrd{\pi_{\txt{guess2}}}$) as both $\Pass{s,x} \in \aspint'$ and $\PassT{s,x} \in \aspint'$ whenever $v(s) = x$ for $x \in \{\tvtt,\tvft\}$. $\aspint'$ also satisfies all groundings of the third rule of $\pi_{\txt{guess2}}$ as whenever $\Pass{s,u} \in \aspint'$ this means that $\adfint(s) = \tvut$ and then $\PassT{s,x} \in \satint{\aspint} \subset \aspint'$ for every $x \in \{\tvtt,\tvft,\tvut\}$.  For the same reason $\aspint'$ also satisfies all possible groundings of the first three rules of $\pi_{\txt{saturate}}$ (contained in  $\ared{\pi_{\txt{saturate}}}{\aspint'} = \agrd{\pi_{\txt{saturate}}}$).  Furthermore, $\aspint'$ satisfies all possible groundings of the last two rules of $\pi_{\txt{saturate}}$ since whenever $\Parg{s} \in \aspint'$ this means that $s \in S$ and then $\PsatT{s} \in \satint{\aspint} \subset \aspint'$ as well as $\PunsatT{s} \in \satint{\aspint} \subset \aspint'$.  Finally, since $\Psaturate \in \satint{\aspint} \subset \aspint'$ the constraint

\begin{align*}
\txt{\derive \Pnot\;\Psaturate.} 
\end{align*}

\noindent is deleted from $\pi_{\pef}(D)$ when forming the reduct $\ared{\pi_{\pef}(D)}{\aspint'}$. We thus have that $\aspint'$ satisfies all of the rules in $\ared{\pi_{\pef}(D)}{\aspint'}$; hence, $\aspint'$ satisfies $\ared{\pi_{\pef}(D)}{\aspint'}$.

Consider now that there is a $\aspint'' \subset \aspint'$ that satisfies $\ared{\pi_{\pef}(D)}{\aspint'}$.  Since $\aspint''$ satisfies $\ared{\pi_{\adm}(D)}{\aspint'} = \ared{\pi_{\adm}(D)}{\aspint}$, we have by the argument in the ``only if'' direction of the proof of Proposition \ref{thm:adm} that $\aspint \subseteq \aspint''$.  Note that then $\Pass{s,x} \in \aspint''$ for every $s \in S$ s.t. $\adfint(s) = x$ for $x \in \{\tvtt,\tvft\}$. On the other hand, $\aspint''$ satisfies the groundings of the first two rules in $\pi_{\txt{guess2}}$ (since $\ared{\pi_{\txt{guess2}}}{\aspint'} = \agrd{\pi_{\txt{guess2}}}$).  It hence follows that also $\PassT{s,x} \in \aspint''$ for every $s \in S$ s.t. $\adfint(s) = x$ for $x \in \{\tvtt,\tvft\}$.  Moreover, since $\aspint''$ satisfies the groundings of the last rule in $\pi_{\txt{guess2}}$ and $\{\Pass{s,\tvut} \mid s \in S, \adfint(s) = \tvut \} \subset \aspint \subset \aspint''$ it must be the case that there is some $x \in \{\tvut,\tvtt,\tvft\}$ s.t. $\PassT{s,x} \in \aspint''$ for every $s \in S$ s.t. $\adfint(s) = \tvut$.
%(in fact, since $\aspint'' \subset \aspint'$ all such atoms that are in $\aspint''$ are also in $\aspint'$).
We thus have that there is an ADF interpretation $\adfint' \geq_i \adfint$ s.t. there is an atom 
$\PassT{s,x} \in \aspint''$ whenever $\adfint'(s) = x$.

Assume now that $\Psaturate \not\in \aspint''$.  Since $\aspint''$ satisfies $\ared{\pi_{\txt{saturate}}}{\aspint'} = \agrd{\pi_{\txt{saturate}}}$ this means that $B(r) \not\subset \aspint''$ for every $r \in \agrd{\pi_{\txt{saturate}}}$.  Hence, in particular,  $B(r) \not\subset \aspint''$ for the rule $r$

\begin{align*}
\Psaturate \derive & \Pass{s_1,\adfint(s_1)},\PassT{s_1,\adfint'(s_1)},
\dots\\
& \Pass{s_k,\adfint(s_k)},\PassT{s_k,\adfint'(s_k)}.
\end{align*}

\noindent This amounts to $\adfint \neq \adfint'$ and, hence, $\adfint <_i \adfint'$.  Also, $B(r) \not\subset \aspint''$ for the rule $r$

\begin{align*}
\Psaturate \derive & \PassT{s,\tvtt}, \PunsatT{s}.
\end{align*}

\noindent for every $s \in S$.  This amounts to (since $\aspint''$ satisfies $\ared{\pi_{\txt{sat2}}(D)}{\aspint'} = \agrd{\pi_{\txt{sat2}}(D)}$; see proof of Proposition \ref{thm:adm}) there not being any $s \in S$ and $w \in \extension{\adfint'}$ for which $\adfint'(s) = \tvtt$ and $w(s) = \tvft$.  In the same manner the fact that $B(r) \not\subset \aspint''$ for the rule $r$

\begin{align*}
\Psaturate \derive & \PassT{s,\tvft}, \PsatT{s}.
\end{align*}

\noindent for every $s \in S$, means that there is no $s \in S$ and $w \in \extension{\adfint'}$ for which $\adfint'(s) = \tvft$ and $w(s) = \tvtt$.  But then $\adfint' \in \adm(D)$ which, together with the fact that $\adfint <_i \adfint'$, is a contradiction to $\adfint \in \pef(D)$. 

On the other hand if $\Psaturate \in \aspint''$, since $\aspint''$ satisfies all possible groundings of the first three rules of $\pi_{\txt{saturate}}$ (as $\ared{\pi_{\txt{saturate}}}{\aspint'} = \agrd{\pi_{\txt{saturate}}}$),  it would be the case that whenever $\Pass{s,\tvut} \in \aspint''$ and hence $v(s) = \tvut$ (since $\aspint \subset \aspint''$) also $\PassT{s,x} \in \aspint''$ for every $x \in \{\tvut,\tvft,\tvtt\}$.  Moreover, if $\Psaturate \in \aspint''$, since $\aspint''$ satisfies all possible groundings of the last two rules of $\pi_{\txt{saturate}}$, it would also follow that $\Psat{s} \in \aspint''$ as well as $\Punsat{s} \in \aspint''$ for every $s \in S$.  This means that if $\Psaturate \in \aspint''$, then $\aspint' \subseteq \aspint''$.  This is a contradiction to our assumption that $\aspint'' \subset \aspint'$.  In conclusion, there is no $\aspint'' \subset \aspint'$ that satisfies $\ared{\pi_{\pef}(D)}{\aspint'}$. Therefore $\aspint' \in \AS(\pi_{\pef}(D))$.

We turn now to proving that for any $\aspint \in \AS(\pi_{\pef}(D))$ it holds that $\adfint \in \pef(D)$ for
$\adfint \cong \aspint$.  Note first of all that since $\aspint$ satisfies $\ared{\pi_{\adm}(D)}{\aspint}$ by the proof of the ``if'' direction of Proposition \ref{thm:adm} we obtain that $\adfint$ is well defined and, moreover, $\adfint \in \adm(D)$.

Since $\adfint \in \adm(D)$, $\adfint \not\in \pef(D)$ would mean that there is a $\adfint' \in \adm(D)$ s.t. $\adfint' >_i \adfint$.  Now, notice first of all that since $\aspint \in \AS(\pi_{\pef}(D))$, $\Psaturate \in \aspint$ since otherwise the constraint

\begin{align*}
\txt{\derive \Pnot\;\Psaturate.} 
\end{align*}

\noindent would not be deleted from $\pi_{\pef}(D)$ (as must be the case) when forming the reduct $\ared{\pi_{\pef}(D)}{\aspint}$ and hence $\pi_{\pef}(D)$ would have no answer set.  We know from the proof of the ``only if'' direction of Proposition \ref{thm:prf} that from $\Psaturate \in \aspint$ it then follows that $\satint{\aspint} \subseteq \aspint$ where

\begin{align*}
\satint{\aspint} =  &\{\PassT{s,x} \mid s \in S, x \in \{\tvtt,\tvft,\tvut\}, \adfint(s) = \tvut\} \ \cup \\
& \{ \PsatT{s} \mid s \in S \} \ \cup \\
& \{ \PunsatT{s} \mid s \in S \} \ \cup \\
& \{ \Psaturate \}.
\end{align*}

\noindent Now let us define

\begin{align*}
\aspint' := &\cup_{p \in \{ \PargP,\PleqP, \PassP, \PsatP, \PunsatP \}} \aspint_{p} \cup \{\PassT{s,v'(s)} \mid s \in S\} \ \cup \\
& \{ \PsatT{s} \mid s \in S, \text{ there is a } w \in \ext{v'} \text{ s.t. } w(\ac{s}) = \tvtt \} \ \cup \\\
& \{ \PunsatT{s} \mid s \in S, \text{ there is a } w \in \ext{v'} \text{ s.t. } w(\ac{s}) = \tvtt \}
\end{align*}

\noindent for which by construction (and $\adfint' >_i \adfint$) $\aspint' \subset \aspint$ holds. Notice first of all that since all negative atoms of $\pi_{\pef}(D)$ occur in $\pi_{\adm}(D) \cup \{\txt{\derive \Pnot\;\Psaturate.}\}$ we have that

\begin{align*}
\ared{\pi_{\pef}(D)}{\aspint} =  \ared{\pi_{\adm}(D)}{\aspint} \cup 
\agrd{\pi_{\txt{guess2}}} \cup
\agrd{\pi_{\txt{sat2}}(D)} \cup
\agrd{\pi_{\txt{check}}(D)} \cup
\agrd{\pi_{\txt{saturate}}}.
\end{align*}

\noindent Now, since $\aspint$ and $\aspint'$ are the same when considering the atoms occurring in $\ared{\pi_{\adm}(D)}{\aspint}$ (meaning: $\cup_{p \in \{ \PargP,\PleqP, \PassP, \PsatP, \PunsatP \}} \aspint_{p} \subset \aspint'$) and $\aspint$ satisfies $\ared{\pi_{\adm}(D)}{\aspint}$ so does $\aspint'$.  Moreover, since $\adfint' >_i \adfint$ by construction $\PassT{s,x} \in \aspint'$ whenever $\Pass{s,x} \in \aspint$ for $x \in \{\tvtt,\tvft\}$ and there is a $y \in  \{\tvut, \tvtt,\tvft\}$ s.t. $\PassT{s,y} \in \aspint'$ whenever $\Pass{s,\tvut} \in \aspint$.  Hence $\aspint'$ also satisfies $\ared{\pi_{\txt{guess2}}}{\aspint} = \agrd{\pi_{\txt{guess2}}}$.

Using analogous arguments as in the ``only if'' direction of the proof of Proposition \ref{thm:adm}, from the fact that $\PassT{s,x} \in \aspint'$ iff $v'(s) = x$ (for $s \in S$ and $x \in \{\tvtt,\tvft,\tvut\}$) and the definition for when $\PsatT{s}$ and $\PunsatT{s}$ are in $\aspint'$, it follows that $\adfint'$ satisfies $\ared{\pi_{\txt{sat2}}}{\aspint} = \agrd{\pi_{\txt{sat2}}}$.  We have also seen in the proof of the ``only if'' direction of Proposition \ref{thm:prf} that $v' \neq v$ and $v' \in \adm(D)$ implies that $\aspint'$ does not satisfy the body of any of the rules in  $\ared{\pi_{\txt{check}}(D)}{\aspint} = \agrd{\pi_{\txt{check}}(D)}$.  Finally, since $\Psaturate \not\in \aspint'$ it is also the case that $\aspint'$ satisfies  $\ared{\pi_{\txt{saturate}}}{\aspint} = \agrd{\pi_{\txt{saturate}}}$.  In conclusion, we have that $\aspint'$ satisfies $\ared{\pi_{\pef}(D)}{\aspint}$ and $\aspint' \subset \aspint$ which contradicts $\aspint \in \AS(\pi_{\pef}(D)))$.  Hence, there cannot be a $v' >_i v$ s.t. $v' \in \adm(D)$ and therefore $\adfint \in \pef(D)$ must be the case.  
\end{proof}

\begin{example}
\label{ex:yadf:prf} The encoding $\pi_{\pef}(D)$ for the ADF $D$ from Example~\ref{ex:adf} as implemented by our system \PNAME{YADF}
%(recall $\ac{a} {=} a \vee \neg b, \ac{b} {=} \neg a$),
%
looks as follows:

\begin{verbatim}
leq(u,0).
leq(u,1).
leq(0,0).
leq(1,1).
arg(a).
arg(b).
arg(c).
asg(S,u) :- arg(S),not asg(S,0),not asg(S,1).
asg(S,0) :- arg(S),not asg(S,1),not asg(S,u).
asg(S,1) :- arg(S),not asg(S,u),not asg(S,0).
sat(a) :- asg(b,Y0),leq(Y0,V0),V1=1-V0,V2=V1?V0,V2=1.
sat(b) :- asg(b,Y0),leq(Y0,V0),V0=1.
sat(c) :- asg(c,Y0),leq(Y0,V0),asg(b,Y1),leq(Y1,V1),
                                  V3=1,V3=V2?V1,V2=1-V0.
inv(a) :- asg(b,Y0),leq(Y0,V0),V1=1-V0,V2=V1?V0,V2=0.
inv(c) :- asg(c,Y0),leq(Y0,V0),asg(b,Y1),leq(Y1,V1),
                                  V3=V2?V1,V3=0,V2=1-V0.
inv(b) :- asg(b,Y0),leq(Y0,V0),V0=0.
:- arg(S),asg(S,1),inv(S).
:- arg(S),asg(S,0),sat(S).
asg2(S,0) :- asg(S,0).
asg2(S,1) :- asg(S,1).
asg2(S,0)|asg2(S,1)|asg2(S,u) :- asg(S,u).
sat2(a) :- asg2(b,Y0),leq(Y0,V0),V1=1-V0,V2=V1?V0,V2=1.
sat2(b) :- asg2(b,Y0),leq(Y0,V0),V0=1.
sat2(c) :- asg2(c,Y0),leq(Y0,V0),asg2(b,Y1),leq(Y1,V1),
                                  V3=1,V3=V2?V1,V2=1-V0.
inv2(b) :- asg2(b,Y0),leq(Y0,V0),V0=0.
inv2(c) :- asg2(c,Y0),leq(Y0,V0),asg2(b,Y1),leq(Y1,V1),
                                  V3=V2?V1,V3=0,V2=1-V0.
inv2(a) :- asg2(b,Y0),leq(Y0,V0),V1=1-V0,V2=V1?V0,V2=0.
saturate :- asg(c,X0),asg2(c,X0),asg(b,X1),asg2(b,X1),
                                  asg(a,X2),asg2(a,X2).
saturate :- arg(S),asg2(S,0),sat2(S).
saturate :- arg(S),asg2(S,1),inv2(S).
asg2(S,u) :- asg(S,u),saturate.
asg2(S,0) :- asg(S,u),saturate.
asg2(S,1) :- asg(S,u),saturate.
sat2(S) :- arg(S),saturate.
inv2(S) :- arg(S),saturate.
:- not saturate.
\end{verbatim}

An output of an ASP solver given this instance looks as follows (only showing $\PassP$ and $\PsaturateP$ predicates):
%$\PsatP$, and $\PunsatP$ predicates):
\begin{verbatim}
Answer: 1
asg(c,u) asg(b,0) asg(a,1) saturate 
Answer: 2
asg(c,1) asg(b,1) asg(a,1) saturate
SATISFIABLE
\end{verbatim}
\hfill$\diamond$
%
%Answer: 1
%asg(c,u) asg(b,0) asg(a,1) 
%sat(c) sat(a) unsat(c) unsat(b) ass2(c,1) ass2(a,1) ass2(c,0) ass2(b,0) sat2(c) ass2(c,u) unsat2(c) unsat2(b) saturate sat2(b) sat2(a) unsat2(a)
%Answer: 2
%leq(u,0) leq(0,0) leq(1,1) leq(u,1) arg(c) arg(b) arg(a) ass(c,1) ass(b,1) ass(a,1) sat(b) sat(c) sat(a) ass2(c,1) ass2(b,1) ass2(a,1) sat2(c) unsat2(c) unsat2(b) saturate sat2(b) sat2(a) unsat2(a)
%SATISFIABLE
%
\end{example}

Note that $\pi_{\pef}$, in addition to providing a means of enumerating the preferred interpretations of any ADF, also gives us a complexity adequate means of deciding skeptical acceptance problems.  
%encoding of the enumeration problem for the preferred semantics, gives us complexity adequate encodings for skeptical reasoning.
  The latter via skeptical reasoning of ASP disjunctive programs with predicates of bounded arity (see Table \ref{tbl:complResGrappa}).  Credulous reasoning for the preferred semantics is equivalent to credulous reasoning for the admissible semantics; applying credulous reasoning on the encoding given in Section \ref{sec::admadf} hence provides a means of computation at the right level of complexity for this reasoning task.

\subsection{Encoding for the grounded semantics}
\label{sec:yadf:ground}

Our encoding for the grounded semantics is based on the fact that (see \cite{DBLP:journals/ai/StrassW15}) $v \in \grd(D)$ for an interpretation $v$ and an ADF $D = (S, \{\phi_s\}_{s \in S})$ iff $v$ is the (unique) $\leq_i$-minimal interpretation satisfying

\begin{itemize}
\item for each $s \in S$ such that $v(s) = \tvtt$ there exists an interpretation $w \in [v]_2$ for which $w(\phi_s) = \tvtt$,
\item for each $s \in S$ such that $v(s) = \tvft$ there exists an interpretation $w \in [v]_2$ for which $w(\phi_s) = \tvft$, and
\item for each $s \in S$ such that $v(s) = \tvut$ there exist interpretations $w_1 \in [v]_2$ and  $w_2 \in [v]_2$ such that $w_1(\phi_s) = \tvtt$ and $w_2(\phi_s) = \tvft$.
\end{itemize}

We say that an interpretation $v$ for the ADF $D$ that satisfies exactly one of the above for a specific $s \in S$ (e.g. $v(s) = \tvtt$ and there exists an interpretation $w \in [v]_2$ for which $w(\phi_s) = \tvtt$), that it satisfies the properties for being a candidate for being the grounded interpretation w.r.t. $s$.  The completion $w$, or alternatively the completions $w_1$ and $w_2$, verify this fact for $v$ and $s$.  If $v$ satisfies the properties w.r.t. every $s \in S$ then $v$ is a candidate for being the grounded interpretation of $D$.  An interpretation $v' <_i v$ that is also a candidate for being the grounded interpretation is a counter-model (alternatively, counter-example) to $v$ being the right candidate (for being the grounded interpretation).

Our encoding for the grounded semantics essentially consists first of all, once more in the guessing part $\pi_{\txt{guess}}$ where we guess assignments of truth values to the statements of the  ADF of interest $D$.  This corresponds to guessing an interpretation $v$ for $D$.  Constraints in our encoding filter out guessed interpretations which either are not candidates to being the grounded interpretation or which have counter-models to being the right candidate.  These constraints rely on the rules in $\pi_{\txt{sat}}(D)$ defined in Section~\ref{sec::admadf} and rules defining when an interpretation has a counter-model to being the right candidate respectively.

%\begin{align*}
%\pi_{\txt{lt}} :=& \{\Ple{\tvu,\tvtt}. \Ple{\tvu,\tvft}. \Prop{\tvut,\tvft,\tvtt}. \} 
%\end{align*}

We start with a few facts needed for our encoding.  First of all, we use facts analogous to those in $\pi_{\text{lt}}$ defined in Section~\ref{sec::admadf} for encoding the truth values a possible counter-model to the interpretation guessed via $\pi_{\txt{guess}}$ (being the right candidate for the grounded interpretation) can assign to the statements.  Here we also need an additional argument (the first argument of the predicate $\PlneP$) allowing us to check whether the interpretation in question is distinct from the one determined by the predicate $\PassP$:

\begin{align*}
\pi_{\txt{lne}} :=& \{\Plne{\tvtt,\tvut,\tvtt}. \ \Plne{\tvtt,\tvut,\tvft}.  \} \ \cup \\
                &  \{\Plne{\tvft,\tvtt,\tvtt}. \ \Plne{\tvft,\tvft,\tvft}. \ \Plne{\tvft,\tvut,\tvut}. \}.
\end{align*}

\noindent Given a candidate for the grounded interpretation $v$ determined by the atoms $\PassP$, for instance the two first facts in $\pi_{\txt{lne}}$ express (using the last two arguments of the predicate $\PlneP$) that if for a statement $s$, $v(s) = x$ with $x \in \{\tvtt,\tvft\}$, then a counter-model $v'$ to $v$ being the right candidate (for the grounded interpretation) can map $s$ to the truth value $\tvut$.  Moreover, the first argument of the alluded to facts indicates that in this case $v'(s) \neq v(s)$.  

Secondly, we need a set of facts for checking whether an interpretation satisfies the properties required for candidates to being the grounded interpretation mentioned at the beginning of this section.  Specifically, given a statement $s$ of the ADF $D$, $\Pprop{z_1,z_2,z_3}$ can be used to check whether the correct relationship between $z_1 = v(s)$, $z_2 = w_1(\ac{s})$, and $z_3 = w_2(\ac{s})$ holds for an interpretation $v$ for $D$, and $w_1,w_2 \in \ext{v}$ (e.g. that if $v(s) = \tvut$ then there must be $w_1\in \ext{v}$, $w_2\in \ext{v}$ s.t. $w_1(\ac{s}) = \tvt$ and $w_2(\ac{s}) = \tvf$).  In particular, note that $w_1 = w_2$ is possible and hence $\Pprop{x,y,z}$ can also be used to check the properties for when $v(s) = x$ and $x \in \{\tvtt,\tvft\}$ (first two facts in $\pi_{\txt{prop}}$):      

\begin{align*}
\pi_{\txt{prop}} :=& \{\Pprop{\tvtt,\tvtt,\tvtt}. \ \Pprop{\tvft,\tvft,\tvft}. \ \Pprop{\tvut,\tvft,\tvtt}. \}. 
\end{align*}

The following module consists of constraints checking whether the interpretation corresponding to the assignments guessed via $\pi_{\txt{guess}}$ is a candidate (hence the use of the identifier ``$\txt{ca}$'') for being the grounded interpretation:

\begin{align*}
\pi_{\txt{ca}}(D) := 
&\{  \derive \Parg{S}, \Pass{S,\tvtt}, \Pnot\;\Psat{S}.
\;\,\derive \Parg{S}, \Pass{S,\tvft}, \Pnot\:\Punsat{S}. \} \ \cup \\
&\{\derive \Parg{S}, \Pass{S,\tvu}, \Pnot\;\Punsat{S}. \
\derive \Parg{S}, \Pass{S,\tvu}, \Pnot\;\Psat{S}. \}.
\end{align*}

\noindent The module $\pi_{\txt{ca}}(D)$ assumes, as we stated earlier, that the rules in $\pi_{\txt{sat}}(D)$ (and thus the facts in $\pi_{\txt{lt}}$) defined in Section~\ref{sec::admadf} are also part of the encoding for the grounded semantics. For instance the first constraint then checks that there be no $s \in S$ for which it holds that $v(s) = \tvt$ but there is no $w \in \ext{v}$ for which $w(\ac{s}) = \tvt$ for the interpretation $v$ guessed via the atoms constructed with the predicate $\PassP$.  

Now, note that, since, as we explained before, the grounded interpretation is the minimal (w.r.t. $\leq_i$) of the interpretations that are candidates for being the grounded interpretation, this interpretation can be obtained via choosing the minimal interpretation w.r.t. $\leq_i$ from all interpretations that correspond to some answer set of the encoding

\begin{align*}
&\pi_{\txt{ca-}\ground}(D) := %
\pi_{arg}(D) \cup
\pi_{\txt{lt}} \cup
%\pi_{\txt{lne}} \cup
%\pi_{\txt{prop}} \cup
\pi_{\txt{guess}} \cup
\pi_{\txt{sat}}(D)  \cup
\pi_{\txt{ca}}(D);
\end{align*}

\noindent i.e. what essentially boils down to a skeptical acceptance problem for $\pi_{\txt{ca-}\ground}(D)$.

In order to obtain an encoding not requiring (in the worst case) processing of all answer sets we need a rule defining when an interpretation is a counter-model to the interpretation determined via the predicate $\PassP$ being the right candidate.  For this we will need to make repeated use of the function $\B$ defined in Section~\ref{sec::admadf} within a single rule.  We therefore first of all make the symbol ranging over the ASP-variables representing subformulas of a propositional formula $\phi$ within $\B(\phi)$ an explicit parameter of the function.  This is straightforward, but for completeness we give the full definition of our parametrised version of the function $\B$.  Here $\phi$ is once more a propositional formula built from propositional variables in a set $P$, while now $V$ is an arbitrary (meta-) symbol used to refer to the variables introduced by the function:

\[
\BT{V}(\phi) \hspace{-1.1pt}:=\hspace{-1mm}
    \left\{
        \begin{array}{ll}
                \hspace{-2mm}
		\BT{V}(\phi_1) \cup 	
		\BT{V}(\phi_2) \cup 	
		\{ V_\phi = V_{\phi_1} \texttt{\&} V_{\phi_2} \} &
			\hspace{-2mm}\mbox{if\ } \phi=\phi_1\wedge\phi_2\\
                \hspace{-2mm}
		\BT{V}(\phi_1) \cup 	
		\BT{V}(\phi_2) \cup 	
		\{ V_\phi = V_{\phi_1} \texttt{?} V_{\phi_2} \} &
			\hspace{-2mm}\mbox{if\ } \phi=\phi_1\vee\phi_2\\
                \hspace{-2mm}
		\BT{V}(\psi) \cup 	
		\{ V_\phi = 1 \texttt{-} V_{\psi} \} &
			\hspace{-2mm}\mbox{if\ } \phi=\neg \psi	\\
                \hspace{-2mm}
		\emptyset &
			\hspace{-2mm}\mbox{if\ } \phi= p\in P.
        \end{array}
	\right.
\]
%}

\noindent
Again, $V_\phi$,
$V_{\phi_1}$
$V_{\phi_2}$ and
$V_{\psi}$ are variables representing the subformulas of $\phi$.  From now on, whenever we introduce sets $\BT{V_1}(\phi_1)$ and $\BT{V_2}(\phi_2)$ for possibly identical formulas $\phi_1$ and $\phi_2$ but distinct symbols $V_1$ and $V_2$, we implicitly also assume that then $\BT{V_1}(\phi_1) \cap \BT{V_2}(\phi_2) = \emptyset$.   

Our rule for defining counter-models to an interpretation being the right candidate for the grounded interpretation requires first of all a part for ``generating'' an interpretation less informative (w.r.t $\leq_i$) and distinct from the interpretation determined by $\pi_{\txt{guess}}$; i.e. a candidate counter-model.  For this we use the atoms

\begin{flalign*}
\LE_D := \{ \Pass{s,X_s},\Plne{E_s,Y_s,X_s} \mid s \in S\} \cup \BT{E}(\vee_{s \in S}s) \cup \{E_{\vee_{s \in S}s}= 1\}.
\end{flalign*}

\noindent Here, given an interpretation $v$ determined by the atoms $\Pass{s,X_s}$ ($s \in S$), the atoms $\Plne{E_s,Y_s,X_s}$ are used to generate an assignment of truth values to the statements (via argument $Y_s$) corresponding to an interpretation $v' \leq_i v$.  Then $\BT{E}(\vee_{s \in S}s) \cup \{E_{\vee_{s \in S}s}= 1\}$ are used (via the arguments $E_s$ of the atoms $\Plne{E_s,Y_s,X_s}$) to check that there is a $s \in S$ for which $v'(s) \neq v(s)$ and hence in fact $v' <_i v$.  Thus, if $v$ is a candidate for the grounded interpretation, then $v'$ is a candidate counter-model for $v$ being the grounded interpretation.

%the atoms $\Pass{s,X_s}$ and $\Plne{E_s,Y_s,X_s}$ are used to 
%$\PassP$, $Plne$ are used to generate 

We now introduce the following set of atoms to check whether the candidate counter-model is indeed a counter-model to the interpretation determined by $\PassP$ being the grounded interpretation.  We need to check the properties candidates for being the grounded interpretation need to satisfy for each of the statements $s$ of the ADF of interest $D$; therefore the need for having sets of atoms $\CMl_{s,D}$ defined for every statement $s$:

\begin{flalign*}
\CMl_{s,D} := &\{\Pleq{Y_t,V^{(t,s),1}} \mid t \in \prts{s}\} \cup \BT{V^{(t,s),1}}(\ac{s}) \ \cup \\
          &\{\Pleq{Y_t,V^{(t,s),2}} \mid t \in \prts{s}\} \cup \BT{V^{(t,s),2}}(\ac{s}) \ \cup \\
          & \{\Pprop{Y_s,V^{(t,s),1}_{\ac{s}},V^{(t,s),2}_{\ac{s}}}\}.
\end{flalign*}

\noindent Note here the use of the predicate $\PleqP$ (defined via the module $\pi_{\txt{lt}}$ from Section~\ref{sec::admadf}) for generating assignments to statements corresponding to completions.  The first two lines of the definition of $\CMl_{s,D}$ are used for generating completions  $w_1,w_2 \in \ext{v'}$ of an interpretation $v'$  and the outcome of the evaluation of $\ac{s}$ by the completions.  Then the third line is used to check that $w_1,w_2$ verify that $v'$ is a candidate for being the grounded interpretation (and hence a counter-model for a $v >_i v'$ being the right candidate for the grounded interpretation).

Putting all the above together we have quite a large rule defining when a candidate counter-model is indeed a counter-model to the interpretation determined by $\pi_{\txt{guess}}$ being the right candidate for the grounded interpretation:

%\begin{flalign*}
%\pi_{\txt{cm}}(D)%
%             := \{\Pcm \derive & \LE_D \cup \{\CMl_{s,D} \mid s \in S\}
%             \}.
%\end{flalign*}

\begin{flalign*}
\pi_{\txt{cm}}(D)%
             := \{\Pcm \derive & \LE_D \cup \bigcup_{s \in S} \CMl_{s,D} \}.
\end{flalign*}

\noindent The following is then an encoding allowing to compute the grounded interpretation for the ADF $D$ in one go:

\begin{align*}
&\pi_{\ground}(D) := %
\pi_{arg}(D) \cup
\pi_{\txt{lt}} \cup
\pi_{\txt{lne}} \cup
\pi_{\txt{prop}} \cup
\pi_{\txt{guess}} \cup
\pi_{\txt{sat}}(D) \cup
\pi_{\txt{ca}}(D) \cup
\pi_{\txt{cm}}(D) \cup
\{\derive \Pcm.\}.
\end{align*}

\noindent Note, in particular, the constraint $\{\derive \Pcm.\}$ disallowing interpretations having a counter-model to them being the right candidate for the grounded interpretation.  

\begin{propn}
\label{thm:ground}
For every ADF $D$ it holds that $\ground(D) \cong \AS(\pi_{\ground}(D))$,
\end{propn}

\begin{proof} (sketch) The proof is similar to that of Proposition \ref{thm:stb}.  Indeed note that $\pi_{\stb}(D)$ (defined in Section~\ref{sec:yadf:stable}) essentially builds on $\pi_{\ground}(D)$, the main difference being the slightly more complex versions of $\pi_{\txt{lt}}$, $\pi_{\txt{lne}}$, $\pi_{\txt{prop}}$, $\pi_{\txt{cm}}(D)$ and the use of $\pi_{\txt{model}}(D)$ (see Section~\ref{sec:yadf:stable}) rather than $\pi_{\txt{ca}}(D)$ (and $\pi_{\txt{sat}}(D)$).\end{proof}

We do not obtain complexity sensitive means of deciding credulous and skeptical reasoning \wrt the grounded semantics via the encoding $\pi_{\ground}$. Nevertheless, the encoding offers an alternative strategy to deriving the grounded interpretation of an ADF to that of the static encodings at the basis of the \PNAME{DIAMOND} family of systems mentioned in the introduction to this work.  Also, the encoding forms the basis of the complexity adequate (w.r.t. the reasoning problems) encoding for the stable semantics we present in Section \ref{sec:yadf:stable}.

\subsection{Encoding for the stable semantics}
\label{sec:yadf:stable}

As already indicated and is to be expected given the definition of this semantics, our encoding for the stable semantics is based on the encoding for the grounded semantics.  Nevertheless, some modifications are required.  First of all we need to guess assignments to statements of an ADF $D$ corresponding to a two valued rather than a three valued interpretation $v$ for $D$.  Secondly, we need to check that $v$ is a model of $D$.  Third, we need to ensure that $v$ assigns the truth value $\tvtt$ to the same statements as the grounded interpretation of the reduct of $D^v$ (i.e. $\projection{v}{E_v} = \ground(D^v)$) rather than $D$ simpliciter.

To start we slightly modify some facts used in previous encodings.  Our encoding once more follows the guess \& check methodology and there will therefore be a part used to guess a candidate $v$ for being the stable interpretation of our ADF of interest $D$.  We will then need a modified version of $\pi_{\txt{lt}}$ (defined in Section~\ref{sec::admadf}) to set the right truth values for completions of a candidate counter-model $v'$ to $v$ (actually $\projection{v}{E_v}$) being the right candidate for the grounded interpretation (as explained in Section \ref{sec:yadf:ground}) of the reduct $D^v$:
%Such completions will therefore assign the truth value $\tvft$ to any statement $s$ of $D$ whenever $v(s) = \tvft$ independently of the value of $v'(s)$ (the first argument of the  predicate $\PleqSP$ is used to refer to the value of $v(s)$).      

\begin{align*}
\pi_{\txt{lt}}' :=& \{ \PleqS{\tvtt,\tvut,\tvft}. \ \PleqS{\tvtt,\tvut,\tvtt}. \ \PleqS{\tvtt,\tvft,\tvft}. \ \PleqS{\tvtt,\tvtt,\tvtt}. \} \ \cup \\
              & \{ \PleqS{\tvft,\tvut,\tvft}. \ \PleqS{\tvft,\tvft,\tvft}. \ \PleqS{\tvft,\tvtt,\tvft}.\}.
\end{align*}

%As can be gleaned from our above explanation,
The first four facts in $\pi_{\txt{lt}}'$ are as to those in $\pi_{\txt{lt}}$.  % (defined in Section~\ref{sec::admadf}).
  The only difference is in the first argument which we use to encode the assignment of a truth value to a statement $s$ by a guessed candidate for the grounded interpretation $v$ of the reduct $D^v$.  The other two values express possible values a completion $w$ of a a candidate counter-model $v'$ to $v$ (being the right candidate for the grounded interpretation of $D^v$) can assign to the statement $s$.  For instance, $\PleqS{\tvtt,\tvut,\tvft}$ expresses that if $v(s) = \tvt$ and $v'(s) = \tvut$ then one of the two possible assignments of a truth value to $s$ by $w \in \ext{v'}$ is $w(s) = \tvft$.  The three last facts in $\pi_{\txt{lt}}'$ now indicate possible assignments of truth values to a statement $s$ by a $w \in \ext{v'}$ when $v(s) = \tvft$.  In order to simulate the evaluation of the acceptance conditions of the reduct $D_v$ by the completions of $v'$  in other parts of our encoding we enforce that in this case $w(s) = \tvft$ whatever the value of $v'(s)$.  This amounts to replacing each statement $s$ for which $v(s) = \tvft$ within an acceptance condition $\ac{s'}$ in which the statement $s$ occurs by $\bot$ as is required by the definition of the reduct.    

The module $\pi_{\txt{lne}}$ defined in Section \ref{sec:yadf:ground} also needs to be modified (by one fact) to account for the fact that a counter-model $v'$ to an interpretation $v$ being the right candidate for satisfying $\projection{v}{E_v} = \grd(D^v)$
%$\projection{v}{E_v}$ (where $v$ is the interpretation guessed to be stable) being the right candidate for the grounded interpretation of $D^v$
must be distinct from $v$ on the statements assigned the truth value $\tvtt$ (i.e. there must be at least one statement $s$ to which $v$ assigns the truth value $\tvtt$ and $v'$ the truth value $\tvut$): 

\begin{align*}
\pi_{\txt{lne}}' :=& \{\Plne{\tvtt,\tvut,\tvtt}. \ \Plne{\tvft,\tvut,\tvft}.  \} \ \cup \\
                &  \{\Plne{\tvft,\tvtt,\tvtt}. \ \Plne{\tvft,\tvft,\tvft}. \ \Plne{\tvft,\tvut,\tvut}. \}.
\end{align*}

\noindent The difference of $\pi_{\txt{lne}}'$ w.r.t $\pi_{\txt{lne}}$ is thus in the second fact ``$\Plne{\tvft,\tvut,\tvft}$.'' where the first argument in the corresponding fact in $\pi_{\txt{lne}}$ is $1$ rather than $0$.  

We also need to modify $\pi_{\txt{prop}}$ defined in Section \ref{sec:yadf:ground} adding an extra-argument (again, the first one) to indicate whether the property required per statement of an ADF for candidates for the grounded interpretation is verified or not:      

\begin{align*}
\pi_{\txt{prop}}' :=& \{\PpropS{\tvtt,\tvtt,\tvtt,\tvtt}. \ \PpropS{\tvtt,\tvft,\tvft,\tvft}. \ \PpropS{\tvtt,\tvut,\tvft,\tvtt}. \ \PpropS{\tvtt,\tvut,\tvtt,\tvft}.  \} \ \cup \\
&\{\PpropS{\tvft,\tvtt,\tvft,\tvtt}. \ \PpropS{\tvft,\tvtt,\tvtt,\tvft}. \ \PpropS{\tvft,\tvtt,\tvft,\tvft}. \} \ \cup \\
&\{\PpropS{\tvft,\tvft,\tvft,\tvtt}. \ \PpropS{\tvft,\tvft,\tvtt,\tvft}. \ \PpropS{\tvft,\tvft,\tvtt,\tvtt}. \} \ \cup \\
&\{\PpropS{\tvft,\tvut,\tvft,\tvft}. \ \PpropS{\tvft,\tvut,\tvtt,\tvtt}.\}.
\end{align*}

\noindent Here for an ADF of interest (in our encoding, the reduct $D^v$ of the interpretation $v$ guessed to be stable) we list all possible combinations of truth values of $v(s)$, $w_1(\ac{s})$, $w_2(\ac{s})$ for $w_1,w_2 \in \ext{v}$ (three last arguments in the facts) and indicate (first argument in the facts) whether the combination in question makes $w_1,w_2$ witnesses of $v$ being a candidate for the grounded interpretation w.r.t. the statement $s$ (as explained in Section \ref{sec:yadf:ground}). For instance $\PpropS{\tvtt,\tvft,\tvft,\tvft}$ indicates that $w_1(\phi_s) = w_2(\phi_s) = \tvft$ makes $w_1,w_2$ witnesses of $v$ being a candidate (w.r.t. $s$) when $v(s) = \tvft$.

As already indicated, also our encoding for the stable semantics builds on a module guessing possible assignments to the statements of the ADF $D$.  We only need to slightly modify $\pi_{\txt{guess}}$ as defined in Section~\ref{sec::admadf} to obtain a conjecture for the stable interpretation corresponding to a two valued rather than three valued interpretation for $D$:

\begin{flalign*}
\pi_{\txt{guess}}'
			:= \{ \Pass{S,\tvft} \derive & \Pnot\;\Pass{S,\tvtt}, \Parg{S}.\\
                            \Pass{S,\tvtt} \derive & \Pnot\;\Pass{S,\tvft}, \Parg{S}. \}.
\end{flalign*}

In order to check that the guessed interpretation is a model of $D$ we again need to evaluate the acceptance conditions of $D$ but this time by the guessed interpretation.  For this we make use of the following sets of atoms per statement $s$ of $D$:

\begin{flalign*}
\Ml_s := \{ \Pass{t,V_t} \mid t \in \prts{s}\} \cup \B(\ac{s}).
\end{flalign*}

\noindent Here, we again make use of the function $\B$ defined in Section~\ref{sec::admadf} but this time to evaluate the acceptance condition $\ac{s}$ by the interpretation guessed to be stable (and thus a model) of $D$ via $\pi_{\txt{guess}}'$. 

\noindent The following are then constraints, one per statement, filtering out guesses that are not models of $D$:

\begin{flalign*}
\pi_{\txt{model}}(D)%
			:= \{\derive & \ \Pass{s,V_s},\Ml_s, V_s \aneq V_{\ac{s}}. \mid s \in S\}.
\end{flalign*}

\noindent More to the point, the constraints filter out any guessed interpretation $v$ (via $\pi_{\txt{guess}}'$), for which $v(s) \neq v(\phi_s)$ for some statement $s$.

Now, note that for any $v \in \mdl(D)$, $\projection{v}{E_v}$ is a candidate to being the grounded interpretation of the reduct $D^v$.  The reason is first of all that $\ext{\projection{v}{E_v}} = \{\projection{v}{E_v}\}$ and, hence, for any $w \in \ext{\projection{v}{E_v}}$, $w(\ac{s}') = \projection{v}{E_v}(\ac{s}') = v(\ac{s})$ for the modified acceptance conditions $\ac{s}' = \phi_s[b/\bot : v(b) = \tvft]$ of $D^v$.  As a consequence, clearly whenever $\projection{v}{E_v}(s) = x$ for $x \in \{\tvtt,\tvft\}$ (in fact, $x = \tvtt$) there is a $w \in \ext{v}$, namely $w = \projection{v}{E_v}$, for which $w(\ac{s}) = x$.  In effect, the latter is the case by virtue of $v \in \mdl(D)$ and hence $v(\ac{s}) = \projection{v}{E_v}(\ac{s}') = x$  whenever $v(s) = x$. Also, there are no statements for which $\projection{v}{E_v}(s) = \tvut$.  The consequence for our encoding for the stable semantics is that $\pi_{\txt{model}}(D)$ suffices for checking whether our guessed interpretation, when projected on the statements to which it assigns the truth value $\tvtt$, is a candidate for being the grounded interpretation of $D^v$.

%For the encoding of the stable semantics we make use of Lemma \ref{lemma-altred} from Section \ref{}.  The lemma states that, in order to check whether an interpretation $v$ coincides with the grounded interpretation of the reduct $D^v$ on the statements to which it assigns the truth value $\tvtt$ one can equivalently check the same fact for a modified form of the reduct $\altred := (S,C^* := \{\phi^*_s\}_{s \in S})$ with $\phi^*_s := \phi_s$ if $v(s) = \tvtt$ and $\phi^*_s := \bot$ if $v(s) = \tvft$.

%Now, as we already noted in Section \ref{}, 
%any $v \in \mdl(D)$ is a candidate to being the grounded interpretation of the ADF $\altred = (S,C^* := \{\phi^*_s\}_{s \in S})$.  
%In effect,
%if $s$ is a statement of $\altred$ s.t. $v(s) = \tvt$ then there is an interpretation $w = v \in \extension{v}$ s.t. $w(\phi^*_s) = w(\phi_s) = \tvt$ (since $v \in \mdl(D)$, $v(s) = v(\phi_s)$).  On the other hand, if $s$ is a statement of $\altred$ s.t. $v(s) = \tvf$ then there is an interpretation $w = v \in \extension{v}$ s.t. $w(\phi^*_s) = w(\bot) =  \tvf$ ($w(\bot) = \tvf$ by the semantics of $\bot$).  Also, for no statement $s$ of $\altred$ $v(s) = \tvu$, since $v$ is a two-valued interpretation.  The consequence for our encoding of for the stable semantics is that $\pi_{\txt{model}}(D)$ suffices for checking whether our guessed interpretation is a candidate for being the grounded interpretation of $\altred$; i.e. also $D^v$.

All that remains for our encoding of the stable semantics is therefore, as we have for the encoding of the grounded semantics, a constraint filtering out guessed interpretations which have counter-models to being the right candidate for being the grounded interpretation of $D^v$.  For this we introduce a slightly modified version of  $\pi_{\txt{cm}}(D)$ (defined in Section \ref{sec:yadf:ground}) accounting for the fact that we need to check for counter-models to $\projection{v}{E_v}$ being the right candidate for the reduct $D^v$ rather than $v$ and $D$.  This means that completions of potential counter-models need to set any statement set to the truth value $\tvft$ by $v$ also to $\tvft$.  To encode this we use the predicate $\PleqSP$ rather than $\PleqP$ in our modified version $\CMl_{s,D}'$ of the set of atoms $\CMl_{s,D}$ (from Section \ref{sec:yadf:ground}). Also, we need to check the properties that candidates of the grounded interpretation need to satisfy only for statements $s$ for which $v(s) = \tvtt$.   To encode this we make use of the predicate $\PpropSP$ rather than $\PpropP$ and add a corresponding check using ASP built in boolean arithmetic functions:

\begin{flalign*}
\CMl_{s,D}' := &\{\PleqS{X_t,Y_t,V^{(t,s),1}} \mid t \in \prts{s}\} \cup \BT{V^{(t,s),1}}(\ac{s}) \ \cup \\
          &\{\PleqS{X_t,Y_t,V^{(t,s),2}} \mid t \in \prts{s}\} \cup \BT{V^{(t,s),2}}(\ac{s}) \ \cup \\
          & \{\PpropS{P_s,Y_s,V^{(t,s),1}_{\ac{s}},V^{(t,s),2}_{\ac{s}}}\} \cup \{CX_s = 1 - X_s, O_s = P_s?CX_s, O_s = 1\}.
\end{flalign*}

\noindent Our modified module $\pi_{\txt{cm}}'(D)$ of $\pi_{\txt{cm}}(D)$ is then as follows:

\begin{flalign*}
\pi_{\txt{cm}}'(D)%
             := \{\Pcm \derive & \LE_D \cup \bigcup_{s \in S} \CMl_{s,D}' \}.
\end{flalign*}

\noindent Note that we here make use of the set of atoms $\LE_D$ as defined in Section \ref{sec:yadf:ground}, yet relying on the definition of the predicate $\PlneP$ as given by the module $\pi_{\txt{lne}}'$ rather than $\pi_{\txt{lne}}$. Putting everything together the encoding for the stable semantics has the following form:

\begin{align*}
\pi_{\stb}(D) := %
&\pi_{arg}(D) \ \cup \
\pi_{\txt{lt}}' \ \cup \
\pi_{\txt{lne}}' \ \cup \
\pi_{\txt{prop}}' \ \cup \\
&\pi_{\txt{guess}}' \ \cup \
\pi_{\txt{model}}(D) \ \cup \
\pi_{\txt{cm}}'(D) \ \cup \
\{\derive \Pcm.\}.
\end{align*}

\begin{propn}
\label{thm:stb}
For every ADF $D$ it holds that $\stable(D) \cong \AS(\pi_{\stable}(D))$.
\end{propn}

\begin{proof}Let $\adf$ be an ADF and $v \in \stable(D)$.  Let also

\begin{align*}
\aspint := &\pi_{arg}(D) \ \cup \ \pi_{\txt{lt}}' \ \cup \ \pi_{\txt{lne}}' \ \cup \ \pi_{\txt{prop}}' \ \cup \ \{\Pass{s,x} \mid s \in S, v(s) = x \}  
\end{align*}

\noindent be a set of ground atoms (such that $v \cong \aspint$). (We slightly abuse the notation here by using e.g. $\pi_{arg}(D)$ to refer to the set of atoms rather than the facts in the module.)   We prove now that $\aspint \in \AS(\pi_{\stable}(D))$.

We start by proving that $\aspint$ satisfies $\ared{\pi_{\stable}(D)}{\aspint}$.  Note first that $\aspint$ satisfies each of $\ared{\pi_{arg}(D)}{\aspint} = \pi_{arg}(D)$, $\ared{\pi_{\txt{lt}}'}{\aspint} = \pi_{\txt{lt}}'$, $\ared{\pi_{\txt{lne}}'}{\aspint} = \pi_{\txt{lne}}'$, $\ared{\pi_{\txt{prop}}'}{\aspint} = \pi_{\txt{prop}}'$ since all of the facts in each of these modules are in $\aspint$.  $\aspint$ also satisfies

\begin{align*}
\ared{\pi_{\txt{guess}}'}{\aspint} = \{\Pass{s, x} \derive \Parg{s}. \mid s \in S,
\Pass{s, y} \not\in \aspint, x \in \{\tvtt, \tvft\}, y \in (\{\tvtt, \tvft\} \ \{x\})\}
\end{align*}

\noindent by the fact that $\adfint \cong \aspint$.  

Assume now that $\aspint$ satisfies the body of some constraint in $\ared{\pi_{\txt{model}}(D)}{\aspint}$, i.e. there is a $s \in S$ and a substitution $\theta$ s.t. $\Pass{s,\theta(V_s)} \in \aspint$, $\Pass{t,\theta(V_t)} \in \aspint$ for each $t \in \prts{s}$, $\theta(\B(\ac{s})) \in \aspint$, and $\theta(V_s \neq V_{\ac{s}}) \in \aspint$.  This translates to $v(s) \neq v(\ac{s})$ which means $\adfint \not\in \mdl(D)$ and contradicts $v \in \stable(D)$.  Therefore $\aspint$ does not satisfy any of the constraints in $\ared{\pi_{\txt{model}}(D)}{\aspint}$.

Consider on the other hand that $\aspint$ satisfies the body of some rule in $\ared{\pi_{\txt{cm}}'(D)}{\aspint}$.  This means that there is a substitution $\theta$ such that first of all $\Pass{s,\theta(X_s)} \in \aspint$ as well as $\Plne{\theta(E_s),\theta(Y_s),\theta(X_s)} \in \aspint$ for every $s \in S$.  Also $\theta(\BT{E}(\vee_{s \in S}s)) \in \aspint$ and $\theta(E_{\vee_{s \in S}s}= 1) \in \aspint$.  All of this together means that $v'(s) <_i v(s)$ for the interpretation $v'$ defined as $v'(s) := \theta(Y_s)$.  Moreover, since $\aspint$ satisfies $\pi_{\txt{lt}}'$, there is an $s \in S$ s.t. $v(s) = \tvtt$ and $v'(s) = \tvut$.  This means that also $E_v \neq \emptyset$ and  $\projection{v'}{E_v}(s) <_i \projection{v}{E_v}(s)$ where $\projection{v}{E_v}$ and $\projection{v'}{E_v}$ are interpretations of the reduct $D^v$.%; i.e. $\projection{v}{E_v} = v(s)$ and $\projection{v'}{E_v}(s) = v'(s)$ for $s \in E_v$.   

Secondly, for every $s \in S$ we have that $\PleqS{\theta(X_t),\theta(Y_t),\theta(V^{(t,s),1})} \in \aspint$ and it is also the case that $\PleqS{\theta(X_t),\theta(Y_t),\theta(V^{(t,s),2})} \in \aspint$ for every  $t \in \prts{s}$.  Consider the interpretations $w_i$ for $i \in \{1,2\}$ defined as $w_i(t) = \theta(V^{(t,s),i})$ for every $t \in \prts{s}$.  Then $w_i(t) \geq_i v'(t)$ whenever $v(t) = \tvtt$, but $w_i(t) = v(t) = \tvft$ if $v(t) = \tvft$.  This means that $w_i(\ac{s}) = \projection{w_i}{E_v}(\ac{s}')$ for $\ac{s}' = \ac{s}[b/\bot : v(b) = \tvft]$, and $\projection{w_i}{E_v} \in \ext{\projection{v'}{E_v}}$. Now from $\theta(\BT{V^{(t,s),1}}(\ac{s})) \in \aspint$ for every $t \in \prts{s}$, $\theta(\BT{V^{(t,s),2}}(\ac{s})) \in \aspint$ for every $t \in \prts{s}$, and the fact that $\PpropS{\theta(P_s),\theta(Y_s),\theta(V^{(t,s),1}_{\ac{s}}),\theta(V^{(t,s),2}_{\ac{s}}}) \in \aspint$ we have that $\theta(P_s) = \tvtt$ whenever $\projection{w_1}{E_v}$ and $\projection{w_2}{E_v}$ verify that $\projection{v'}{E_v}$ satisfies the properties for being a candidate for the grounded interpretation of $D^v$ w.r.t. $s \in E_v$. Otherwise $\theta(P_s) = \tvft$.  Moreover, from  $\theta(CX_s = 1 - X_s) \in \aspint$, $\theta(O_s = P_s?CX_s) \in \aspint$, and $\theta(O_s = 1) \in \aspint$ it follows that either $\theta(P_s) = \tvtt$ or $\theta(X_s) = v(s) = \tvft$ for every $s \in S$.

In other words, whenever $s \in E_v$ (remember: $E_v \neq \emptyset$) there are completions that verify that $\projection{v'}{E_v}$ satisfies the properties for being a candidate for the grounded interpretation of $D^v$ w.r.t. $s$.  This means that $\projection{v'}{E_v}$ is a counter-model to $\projection{v}{E_v}$ being the grounded interpretation of $D^v$.  This is a contradiction to $v \in \stb(D)$.  Therefore, $\aspint$ does not satisfy the body of any rule in $\ared{\pi_{\txt{cm}}'(D)}{\aspint}$ and, hence, satisfies $\ared{\pi_{\txt{cm}}'(D)}{\aspint}$.  Finally, since $\Pcm \not\in \aspint$, $\aspint$ does not satisfy the body of the constraint $\{\derive \Pcm.\} \in \ared{\pi_{\stb}(D)}{\aspint}$.  In conclusion, $\aspint$ satisfies $\ared{\pi_{\stb}(D)}{\aspint}$.

Now consider any other $\aspint'$ that satisfies $\ared{\pi_{\stb}(D)}{\aspint}$. Clearly, since $\aspint'$ satisfies all of the facts in $\ared{\pi_{\stb}(D)}{\aspint}$, we have that $\pi_{arg}(D) \cup \pi_{\txt{lt}}' \cup \pi_{\txt{lne}}' \cup \pi_{\txt{prop}}' \subseteq \aspint'$.  But also because of the form of $\ared{\pi_{\txt{guess}}'}{\aspint}$ (see above) and the fact that $\aspint'$ satisfies $\ared{\pi_{arg}(D)}{\aspint}$  it must be the case that $\{\Pass{s,x} \mid s \in S, v(s) = x \} \subset \aspint'$.  This means that in addition to $\aspint$ satisfying $\ared{\pi_{\stb}(D)}{\aspint}$ there is also no $\aspint' \subset \aspint$ that satisfies $\ared{\pi_{\stb}(D)}{\aspint}$; i.e. we have that $\aspint \in \AS(\pi_{\stb}(D))$.

We now turn to proving that for any $\aspint \in \AS(\pi_{\stb}(D))$ it holds that $\adfint \in  \stb(D)$ for
$\adfint \cong \aspint$.  Note first that for any such $\aspint$, since $\aspint$ satisfies $\ared{\pi_{arg}(D)}{\aspint} = \pi_{arg}(D)$ and $\ared{\pi_{\txt{guess}}'}{\aspint}$, $\adfint$ s.t. $\adfint \cong \aspint$ is well defined.  Now assume that $\adfint \not\in \stb(D)$.  Then either i) $v \not\in \mdl(D)$ or ii) $v \in \mdl(D)$ but $\projection{v}{E_v} \not\in \ground(D^v)$.

In the first case i) there must be a $s \in S$ s.t. $v(s) \neq v(\ac{s})$.  Consider hence the substitution $\theta$ defined as $\theta(V_s) = v(s)$ and $\theta(V_t) = v(t)$ for $t \in \prts{s}$.  This substitution is s.t. $\theta(B(r)) \subseteq \aspint$ for the constraint in $\pi_{\txt{model}}(D)$ corresponding to $s$.  This would mean that $\aspint$ does not satisfy $\ared{\pi_{\stb}(D)}{\aspint}$ which is a contradiction.  Therefore $v \in \mdl(D)$ which also means that $\projection{v}{E_v}$ is a candidate for the grounded interpretation of $D^v$ (as we argued in detail while explaining our encoding $\pi_{\stb}(D)$).

Consider now the case ii).  Since $\projection{v}{E_v}$ is a candidate for the grounded interpretation of $D^v$ but $\projection{v}{E_v} \not\in \ground(D^v)$ this means there is a counter-model $\projection{v'}{E_v}$ for $\projection{v}{E_v}$ being the right candidate for the grounded interpretation of $D^v$.  First define $v'$ to be s.t. $v'(s) =  \projection{v'}{E_v}(s)$ for $s \in E_v$ while $v'(s) = v(s)$ for $s \not\in E_v$.   Define then the substitution $\theta$ for which $\theta(X_s) = v(s)$ and $\theta(Y_s) = v'(s)$ for every $s \in S$.  Since $v' \neq v$ (because $\projection{v'}{E_v} \neq \projection{v}{E_v}$) there must be an $s \in E_v \subseteq S$ for which $v(s) \neq v'(s)$.  Set $\theta(E_s) = 1$ for all such $s \in S$, but $\theta(E_s) = 0$ whenever $v(s) = v'(s)$.  We thus have that $\theta(\BT{E}(\vee_{s \in S}s)) \in \aspint$ and $E_{\vee_{s \in S}s} = 1 \in \aspint$.  Hence, $\theta(\LE_D) \in \aspint$.

Now, since $\projection{v'}{E_v}$  is a counter-model to $\projection{v}{E_v}$ being the right candidate for the grounded interpretation of $D^v$ we have that for every $s \in E_v$ there are completions $w_{s,1}$ and $w_{s,2}$ of $\projection{v'}{E_v}$ that are witnesses for $\projection{v'}{E_v}$ satisfying the properties candidates for the grounded interpretation neeed to satisfy w.r.t. $s$.  Hence, we continue defining the substitution $\theta$ s.t. $\theta(V^{(t,s),i}) =  w_{s,i}(t)$ for every $s \in E_v$ and $t \in \prts{s} \cap E_v$.  On the other hand $\theta(V^{(t,s),i}) =  \tvft$ for $t \in \prts{s} \setminus E_v$.   Also, $\theta(P_s) = 1$ for every $s \in E_v$.

For $s \not\in E_v$ we on the other hand define  $\theta(V^{(t,s),i}) = w(t)$ ($i \in \{1,2\}$), $t \in \prts{s} \cap E_v$  for some arbitrary $w \in \ext{v'}$.  On the other hand $\theta(V^{(t,s),i}) =  \tvft$ for $t \in \prts{s} \setminus E_v$. Also, $\theta(P_s) = 1$ whenever $v'(s) = w(\ac{s}')$ ($\ac{s}' = \phi_s[b/\bot : v(b) = \tvft]$) and $\theta(P_s) = 0$ otherwise.  Then we have that $\theta(\BT{V^{(t,s),i}}(\ac{s})) \in \aspint$ for $s \in S$, $t \in \prts{s}$, ($i \in \{1,2\}$).  Also, $\theta(CX_s = 1 - X_s) \in \aspint$, $\theta(O_s = P_s?CX_s) \in \aspint$, and $O_s = 1 \in \aspint$ for every $s \in S$ ($\theta(P_s) = 1$ for $s \in E_v$, while $\theta(CX_s) = 1$ for $s \not\in E_v$).  I.e. $\theta(\CMl_{s,D}') \in \aspint$ for every $s \in S$.  Hence, since also $\theta(\LE_D) \in \aspint$, we have that the body of a rule in $\ared{\pi_{\txt{cm}}'(D)}{\aspint}$ is satisfied by $\aspint$ and, therefore, $\Pcm \in \aspint$.  This means that the constraint $\derive \Pcm. \in \ared{\pi_{\stb}(D)}{\aspint}$ is satisfied by $\aspint$ which contradicts $\aspint \in \AS(\pi_{\stb}(D))$.  Therefore, the case ii) is also not possible and $v \in \stb(D)$ must be the case.  \end{proof}

\begin{example} The encoding $\pi_{\stb}(D)$ for the ADF $D$ from Example~\ref{ex:adf} as implemented by our system \PNAME{YADF} (we slightly condense the encoding by generating some facts using rules) looks as follows:

\begin{verbatim}

arg(a).
arg(b).
arg(c).
val(u).
val(0).
val(1).
lt2(1,u,1).
lt2(1,u,0).
lt2(1,0,0).
lt2(1,1,1).
lt2(0,X,0) :- val(X).
lne(1,u,1).
lne(0,u,0).
lne(0,X,X):- val(X).
prop(1,1,1,1).
prop(1,0,0,0).
prop(1,u,1,0).
prop(1,u,0,1).
prop(0,X1,X2,X3) :- val(X1),val(X2),val(X3),not prop(1,X1,X2,X3).
asg(S,1) :- arg(S),not asg(S,0).
asg(S,0) :- arg(S),not asg(S,1).
:- asg(b,V0),V0!=V0.
:- asg(a,V1),asg(b,V0),V3=V2?V0,V3!=V1,V2=1-V0.
:- asg(c,V0),asg(b,V1),V3=V2?V1,V3!=V0,V2=1-V0.
cm :- asg(c,X0),asg(b,X1),asg(a,X2),lne(E0,Y0,XO),lne(E1,Y1,X1),
      lne(E2,Y2,X2),E20=E0?E1,E21=E2?E20,E21=1,lt2(X0,Y0,V1),
      lt2(X1,Y1,V2),V4=V3?V2,V3=1-V1,lt2(X0,Y0,V5),lt2(X1,Y1,V6),
      V7=1-V5,V8=V7?V6,prop(P0,Y0,V4,V8),CX0=1-X0,OR9=P0?CX0,OR9=1,
      lt2(X1,Y1,V10),lt2(X1,Y1,V11),prop(P1,Y1,V10,V11),CX1=1-X1,
      OR12=P1?CX1,OR12=1,lt2(X1,Y1,V13),V15=V14?V13,V14=1-V13,
      lt2(X1,Y1,V16),V17=1-V16,V18=V17?V16,prop(P2,Y2,V15,V18),
      CX2=1-X2,OR19=P2?CX2,OR19=1.
:- cm.
\end{verbatim}

An output of an ASP solver given this instance looks as follows:

\begin{verbatim}
UNSATISFIABLE
\end{verbatim}

\noindent and indicates that the ADF at hand does not possess any stable model.
\hfill$\diamond$
\end{example}

Concluding our presentation of dynamic encodings for ADFs, we note that also $\pi_{\stb}$, in addition to giving us a means of computing the stable interpretations of any ADF, provides us with a complexity-attuned mechanism to decide credulous and skeptical reasoning tasks via the corresponding ASP reasoning tasks (see Table \ref{tbl:complResGrappa}).

\section{Grappa encodings}\label{sec:grapenc}

We now illustrate how to extend the methodology used in our construction of dynamic encodings for ADFs to GRAPPA.  For this purpose we give ASP encodings for the admissible, complete, and preferred semantics.  Reflecting the relationship between ADFs and GRAPPA, structurally the encodings are very similar to those for ADFs; the main difference being in the encoding of the evaluation of the acceptance patterns.   

Also for our encodings for GRAPPA we make use of the correspondence $\cong$ between 3-valued interpretations (now for GRAPPA instances) and sets of ground atoms (interpretations of ASP programs) defined via ASP atoms $\Pass{s,x}$
for statements $s$ and $x \in \truthvals$.
Hence, we now strive for encodings $\pi_{\sigma}$ for $\sigma \in \{\adm,\com,\prf\}$ s.t.\
for every \grappa instance $G$ we get $\sigma(G) \cong \AS(\pi_\sigma(G))$ (see Definition \ref{defcorr} for the formal meaning of the latter overloaded use of $\cong$).  We will reuse several of the ASP fragments we defined for ADFs. Formally this amounts to extending the corresponding encoding functions to also admit \grappa instances as arguments.

%

%In the same manner as the semantics of \grappa
%
%mirror the semantics of ADFs,
%
%the ASP encodings for \grappa instances presented in this section closely resemble the encodings for ADF instances given in
%
%the previous section.
%  The
%
%main
%difference
%
%is the way in which we encode the evaluation of the acceptance patterns, this also clearly being the crucial difference between \grappa and ADF instances.

Throughout this section, let $G=(S,E,L,\lambda,\alpha)$ be a \grappa instance with $S=\{s_1,\dotsc,s_k\}$.  As already hinted at, the main difference between the encodings for GRAPPA and ADFs is in the definition of the set of atoms $\B(\phi)$ (first defined for ADFs in Section~\ref{sec::admadf}) corresponding to the semantic evaluation of the acceptance conditions (now patterns) associated to the statements.   %We redefine the set of atoms $\B(\phi)$ corresponding to
%%
%the semantic evaluation of subformulas of acceptance conditions
%
%
%of ADF statements
%to encode
%
%
%acceptance patterns of \grappa instances.
%
%
%
The recursive function representing
the evaluation of patterns
needs a statement $s$ as an additional parameter and for the encoding of the basic patterns is defined as $\B_s(\phi):=$

%
%
%
%
%{\small
\[
        \begin{array}{ll}
		\B_s(\phi_1) \ \cup \ 	
		\B_s(\phi_2) \ \cup \
		\{ V_\phi = V_{\phi_1} \& V_{\phi_2} \} &
			\mbox{if\ } \phi=\phi_1\wedge\phi_2
		\\
		\B_s(\phi_1) \ \cup \ 	
		\B_s(\phi_2) \ \cup \
		\{ V_\phi = V_{\phi_1} ? V_{\phi_2} \} &
			\mbox{if\ } \phi=\phi_1\vee\phi_2
		\\
		\B_s(\psi) \ \cup \
		\{ V_\phi = 1- V_{\psi} \} &
			\mbox{if\ } \phi=\neg \psi
		\\
	        P_s(\tau) \ \cup \
                \{ V_\phi = \operatorname{\#sum}\{1:V_{\tau} \bar{R} a\} \ \} &
			\mbox{if\ } \phi= \tau R a.
        \end{array}
\]
%}

%
The difference between $\B_s$ and $\B$ (note the missing subscript $s$) as defined in
Section \ref{sec:encadf}
is in the last line where %
$P_s(\tau) \cup \{ V_\phi = \operatorname{\#sum}\{1:V_{\tau} \bar{R} a\} \}$ encodes the evaluation of a basic pattern $\phi= \tau R a$. Here we make use of the ASP aggregate $\operatorname{\#sum}$ as well as the simple function $\bar{R}:=\,\, <=$ (resp.\ $>=$, ${!}{=}$) if $R =\,\, \leq$ (resp.\ $\geq$, $\neq$) and $\bar{R}=R$ otherwise, relating \grappa and ASP syntax for relational operators.

The function $P_s(\tau)$ on the other hand gives us a set of atoms corresponding to the evaluation of a sum  $\tau$ of terms:

%{\small
\[
P_s(\tau) \hspace{-1pt}:=
    \hspace{-2pt}\left\{\hspace{-6pt}
        \begin{array}{ll}
        P_s(\chi) \ \cup \ T_s(t) \ \cup \
        \{ V_{\tau} = a*V_t + V_{\chi} \}
& \hspace{-7pt} \mbox{if\ } \tau = at + \chi \\
        T_s(t) \ \cup \
        \{ V_{\tau} = a*V_t\}
& \hspace{-7pt} \mbox{if\ } \tau = at. \\
        \end{array}
      \right.
\]
%}

The definition of $P_s$ in turn makes use of the function $T_s(t)$
that returns an atom representing a term $t$.
Here let $s \in S$ be fixed and
$\parent(s) = \{r_1,\ldots,r_q\}$,
$l_r = \lambda(r,s)$ for $r \in \parent(s)$,
and $\parent(s,l) = \{r \in \parent(s) \mid l_r = l\}$.
In order to define atoms corresponding to the evaluation of terms
depending
on the active labels
(those without subscript $t$)
we use the ASP aggregates
$\operatorname{\#sum}$, $\operatorname{\#min}$, $\operatorname{\#max}$, and $\operatorname{\#count}$,
as well as variables $Z_r$ corresponding to completions of the guessed assignment of statements $r \in S$.
Atoms corresponding to terms whose evaluation is independent of the active labels, on the other hand,
can be constructed based on the instance $G$ only.
We define $T_s(t)$ as

%
%{\small
%
\begin{flalign*}
& \{ \ V_t = \operatorname{\#sum}\{Z_{r_{i_1}},r_{i_1} \mtt{;}\dots\mtt{;} Z_{r_{i_w}},r_{i_w}\} \  \} \\%[-3pt]
& \mbox{ with } \{r_{i_1},\ldots,r_{i_w}\} = \parent(s,l)
\hspace{17pt} \mbox{if\ } t = \# l \mbox{ and } \parent(s,l) \neq \emptyset\\%[1pt]
& \{ \ V_t = N \ \} \mbox{ with } N = |\parent(s,l)|
\hspace{75pt} \mbox{if\ } t = \#_t l \\%[1pt]
& \{ \ V_t = \operatorname{\#min}\{l_{r_1}:Z_{r_1}=1\mtt{;}\dots\mtt{;}l_{r_q}:Z_{r_q}=1\} \ \} \\%[-3pt]
& \hspace{195pt} \mbox{if\ } t = \operatorname{min} \\%[1pt]
& \{ \ V_t = N \ \} \mbox{ with } N = min\{l_{r_1},\dots,l_{r_q}\}
\hspace{40pt} \mbox{if\ } t = \operatorname{min_t} \\%[1pt]
& \{ \ V_t = \operatorname{\#max}\{l_{r_1}:Z_{r_1}=1\mtt{;}\dots\mtt{;}l_{r_q}:Z_{r_q}=1\} \ \}
\hspace{2pt} \mbox{if\ } t = \operatorname{max} \\%[1pt]
& \{ \ V_t = N \ \} \mbox{ with } N = max\{l_{r_1},\ldots,l_{r_q}\}
\hspace{37pt} \mbox{if\ } t = \operatorname{max_t} \\%[1pt]
& \{ \ V_t = \operatorname{\#sum}\{l_{r_1},r_1:Z_{r_1}=1\mtt{;}\dots\mtt{;}l_{r_q},r_q:Z_{r_q}=1\} \ \} \\%[-3pt]
& \hspace{132pt} \mbox{if\ } t = \operatorname{sum} \mbox{ and } \parent(s) \neq \emptyset\\%[1pt]
& \{ \ V_t = N \ \} \mbox{ with } N = l_{r_1} + \ldots + l_{r_q}
\hspace{52pt} \mbox{if\ } t = \operatorname{sum_t} \\%[1pt]
& \{ \ V_t = \operatorname{\#count}\{l_{r_1}:Z_{r_1}=1\mtt{;}\dots\mtt{;}l_{r_q}:Z_{r_q}=1\} \ \} \\%[-3pt]
& \hspace{127pt} \mbox{if\ } t = \operatorname{count} \mbox{ and } \parent(s) \neq \emptyset\\%[1pt]
& \{ \ V_t = N \ \} \mbox{ with } N = |\{l_r \mid r \in \parent(s)\}|
\hspace{32pt} \mbox{if\ } t = \operatorname{count_t} \\%[1pt]
& \{ \ V_t = 0  \}
\hspace{53pt} \mbox{if\ } t = \# l \mbox{ and } \parent(s,l) = \emptyset \\%[-3pt]
& \hspace{90pt} \mbox{or } t = \operatorname{sum}, t = \operatorname{count} \mbox{ and } \parent(s) = \emptyset.
\end{flalign*}
%}
%
\label{sec::admg}

For instance the first atom $V_t = \operatorname{\#sum}\{Z_{r_{i_1}},r_{i_1} \mtt{;}\dots\mtt{;} Z_{r_{i_w}},r_{i_w}\}$ corresponds to the computation of $val_o^s(\#l)$ when $\parent(s,l) = \{r_{i_1},\ldots,r_{i_w}\} \neq \emptyset$. Here $o= m^z_s$ is the multi-set of active labels of $s$ under $z$ and the assignments of truth values to statements of the ADF interpretation $z$ is captured via the variables $Z_{r_{i_j}}$ ($1 \leq j \leq z$).  An assumption of the encoding is that such variables $Z_{r_{i_j}}$ take only values $\tvt$ or $\tvf$, thus corresponding to two valued interpretations (e.g. completions).  This, as in the encodings for ADFs, is taken care of in the rules in which the atoms $\B_s$ occur (see the re-definition of the modules  $\pi_{\txt{sat}}$ and  $\pi_{\txt{sat2}}$ defined in Section  \ref{sec:encadf} for GRAPPA instances below).  

As pointed out earlier,
the encodings for \grappa instances
for $\sigma \in \{\adm,\com,\prf\}$
differ from the corresponding
\adf encodings
only in the fragments handling the evaluation of the acceptance patterns (under the completions of an interpretation). Hence, the encodings
$\pi_{\sigma}(G)$ for the \grappa instance $G$ boil down to the programs

%{\small
\begin{flalign*}
\pi_{\adm}(G):=& \ \pi_{0}(G) \ \cup \ \pi_{\txt{guess}} \ \cup \ \pi_{\txt{sat}}'(G) \ \cup \\
&  \{ \derive \Parg{S}, \Pass{S,\tvt}, \Punsat{S}.\,
\derive \Parg{S}, \Pass{S,\tvf}, \Psat{S}. \}; \\%[1pt]
\pi_{\com}(G) :=& \ \pi_{\adm}(G) \ \cup \
\{ \derive \Parg{S}, \Pass{S,\tvu}, \Pnot\;\Punsat{S}.\\
&  \derive \Parg{S}, \Pass{S,\tvu}, \Pnot\;\Psat{S}. \}; \\%[1pt]
\pi_{\prf}(G) :=& \ \pi_{\adm}(G) \ \cup \ \pi_{\txt{guess2}} \ \cup \ \pi_{\txt{sat2}}'(G) \ \cup \ \pi_{\txt{check}}(G) \ \cup \
\pi_{\txt{saturate}} \ \cup \\
&\{ \derive \Pnot\;\Psaturate. \}.
\end{flalign*}
%}

\begin{comment}

{\small
\begin{flalign*}
\pi_{\adm}(G):=& \pi_{0}(G)  \cup
\pi_{\txt{guess}}(G) \cup
\pi_{\txt{sat}}'(G) \cup \\
&  \{  \txt{ :- arg(S), ass(S,\tvt), unsat(S).} \\
&  \txt{ :- arg(S), ass(S,\tvf), sat(S).} \}
\end{flalign*}
}

\end{comment}
%
Here, the difference to the encoding for ADF semantics is the use of the program fragments

%{\small
\begin{align*}
\pi_{\txt{sat}}'(G) := \{
& \Psat{s} \derive
\Bl_s, V_{\alpha(s)} = 1.\\%[-1pt]
&
\Punsat{s} \derive
\Bl_s, V_{\alpha(s)} = 0.
\mid
s\in S
\}; \\[1pt]
\pi_{\txt{sat2}}'(G) := \{
& \PsatT{s} \derive
\Bl2_s, V_{\alpha(s)} = 1. \\%[-1pt]
\,\,
& \PunsatT{s} \derive
\Bl2_s, V_{\alpha(s)} = 0.
\mid
s\in S
\}
\end{align*}
%}

\noindent
where we make use of the shortcuts $\Bl_s$ and $\Bl2_s$.
In their definitions %
we in turn
use
the function
$\B_s(\phi)$ returning the %
atoms for evaluating a \grappa acceptance function:
%

%{\small
\begin{align*}
\Bl_s :=& \{ \Pass{r,Y_r}, \Pleq{Y_r,Z_r} \mid
r\in\parent(s)\} \ \cup \ \B_s(\alpha(s)); \\
\Bl2_s :=& \{ \PassT{r,Y_r}, \Pleq{Y_r,Z_r} \mid
r\in\parent(s)\} \ \cup \ \B_s(\alpha(s)).
\end{align*}
%}

\begin{propn}
\label{thm:admg}
\label{thm:prfg}
\label{thm:comg}
For $\sigma\in\{\adm,\com,\prf\}$
it holds for every \grappa instance $G$
that $\sigma(G) \cong \AS(\pi_\sigma(G))$.
\end{propn}

\begin{proof}(sketch)
The proofs are exactly as those of Propositions \ref{thm:adm}, \ref{thm:com}, and \ref{thm:prf}; they differ only in the parts in which reference is made to the encoding of the evaluation of the acceptance patterns.   
\end{proof}

%\begin{comment}
\begin{example}
\label{ex:adm_grappa}
Consider the \grappa instance $G$ with
$S = \{a,b,c,d\}$, $E = \{(b,b),(a,c),(b,c),(b,d)\}$, $L = \{+,-\}$, $\lambda((b,b)) = +$,
$\lambda((a,c)) = +$, $\lambda((b,c)) = +$, $\lambda((b,d)) = -$, $\pi(s) = \#_t(+) - \#(+) = 0 \wedge \#(-) = 0$ for every $s \in S$.

The encoding $\pi_\adm(G)$ is as follows:

%\begin{scriptsize}
\begin{verbatim}
arg(a). arg(b). arg(c). arg(d).
lt(u,0). lt(u,1). lt(0,0). lt(1,1).

asg(S,0) :- not asg(S,1), not asg(S,u), arg(S).
asg(S,1) :- not asg(S,u), not asg(S,0), arg(S).
asg(S,u) :- not asg(S,0), not asg(S,1), arg(S).


sat(a) :-  Vbp1s2t = 0,Vbp1s2 = (-1)*Vbp1s2t,
           Vbp1s1t = 0, Vbp1s1 = Vbp1s1t + Vbp1s2,
           Vbp2s1t = 0, Vbp2s1 = Vbp2s1t,
           Vbp1 = #sum{1: Vbp1s1 = 0},
           Vbp2 = #sum{1: Vbp2s1 = 0}, Vp=Vbp1&Vbp2, Vp=1.
unsat(a) :-  Vbp1s2t = 0,Vbp1s2 = (-1)*Vbp1s2t,
             Vbp1s1t = 0, Vbp1s1 = Vbp1s1t + Vbp1s2,
             Vbp2s1t = 0, Vbp2s1 = Vbp2s1t,
             Vbp1 = #sum{1: Vbp1s1 = 0},
             Vbp2 = #sum{1: Vbp2s1 = 0}, Vp=Vbp1&Vbp2, Vp=0.
sat(b) :- asg(b,Y_b),lt(Y_b,Z_b),
          Vbp1s2t = #sum{Z_b,b},Vbp1s2 = (-1)*Vbp1s2t,
          Vbp1s1t = 1, Vbp1s1 = Vbp1s1t + Vbp1s2,
          Vbp2s1t = 0, Vbp2s1 = Vbp2s1t,
          Vbp1 = #sum{1: Vbp1s1 = 0},
          Vbp2 = #sum{1: Vbp2s1 = 0}, Vp=Vbp1&Vbp2, Vp=1.
unsat(b) :- asg(b,Y_b),lt(Y_b,Z_b),
            Vbp1s2t = #sum{Z_b,b},Vbp1s2 = (-1)*Vbp1s2t,
            Vbp1s1t = 1, Vbp1s1 = Vbp1s1t + Vbp1s2,
            Vbp2s1t = 0, Vbp2s1 = Vbp2s1t,
            Vbp1 = #sum{1: Vbp1s1 = 0},
            Vbp2 = #sum{1: Vbp2s1 = 0}, Vp=Vbp1&Vbp2, Vp=0.
sat(c) :- asg(a,Y_a),lt(Y_a,Z_a),asg(b,Y_b),
          lt(Y_b,Z_b),Vbp1s2t = #sum{Z_a,a;Z_b,b},
          Vbp1s2 = (-1)*Vbp1s2t,Vbp1s1t = 2, 
          Vbp1s1 = Vbp1s1t + Vbp1s2, Vbp2s1t = 0,
          Vbp2s1 = Vbp2s1t, Vbp1 = #sum{1: Vbp1s1 = 0},
          Vbp2 = #sum{1: Vbp2s1 = 0}, Vp=Vbp1&Vbp2, Vp=1.
unsat(c) :- asg(a,Y_a),lt(Y_a,Z_a),asg(b,Y_b),
            lt(Y_b,Z_b),Vbp1s2t = #sum{Z_a,a;Z_b,b},
            Vbp1s2 = (-1)*Vbp1s2t,Vbp1s1t = 2,
            Vbp1s1 = Vbp1s1t + Vbp1s2, Vbp2s1t = 0, 
            Vbp2s1 = Vbp2s1t, Vbp1 = #sum{1: Vbp1s1 = 0},
            Vbp2 = #sum{1: Vbp2s1 = 0}, Vp=Vbp1&Vbp2, Vp=0.
sat(d) :- asg(b,Y_b),lt(Y_b,Z_b),
          Vbp1s2t = 0,Vbp1s2 = (-1)*Vbp1s2t,
          Vbp1s1t = 0, Vbp1s1 = Vbp1s1t + Vbp1s2,
          Vbp2s1t = #sum{Z_b,b}, Vbp2s1 = Vbp2s1t,
          Vbp1 = #sum{1: Vbp1s1 = 0},
          Vbp2 = #sum{1: Vbp2s1 = 0}, Vp=Vbp1&Vbp2, Vp=1.
unsat(d) :- asg(b,Y_b),lt(Y_b,Z_b),
            Vbp1s2t = 0,Vbp1s2 = (-1)*Vbp1s2t,
            Vbp1s1t = 0, Vbp1s1 = Vbp1s1t + Vbp1s2,
            Vbp2s1t = #sum{Z_b,b}, Vbp2s1 = Vbp2s1t,
            Vbp1 = #sum{1: Vbp1s1 = 0}, 
            Vbp2 = #sum{1: Vbp2s1 = 0}, Vp=Vbp1&Vbp2, Vp=0.

:- arg(S), asg(S,1), unsat(S).
:- arg(S), asg(S,0), sat(S).
\end{verbatim}
%\end{scriptsize}
\hfill$\diamond$
\end{example}
%\end{comment}

%
\label{sec::comg}

\begin{comment}
The only difference between the encoding of the complete semantics $\pi_{\com}$ for \grappa instances and the corresponding encoding for ADFs is again the use of the modified version of the program fragment $\pi_{\txt{sat}}$, $\pi_{\txt{sat}}'$.

{\small
\begin{flalign*}
\pi_{\com}(G):=& \pi_{0}(G) \cup
\pi_{\txt{guess}}(G) \cup
\pi_{\txt{sat}}'(G) \cup \\
&  \{  \txt{ :- arg($S$), ass($S$,$\tvt$), unsat($S$).} \\
&  \txt{ :- arg($S$), ass($S$,$\tvf$), sat($S$).} \\
& \txt{ :- arg($S$), ass($S$,$\tvu$), not unsat($S$).} \\
& \txt{ :- arg($S$), ass($S$,$\tvu$), not sat($S$).} \}
\end{flalign*}
}
\begin{theorem}
For every \grappa instance $G$ it holds that $\com(G) \cong \AS(\pi_\com(G))$.
\end{theorem}
\end{comment}

%
\label{sec::prefg}

\begin{comment}
Finally, also the encoding $\pi_{\prf}(G)$ for the \grappa instance $G$ mirrors the corresponding encoding for ADFs.

{\small
\begin{flalign*}
\pi_{\prf}(G):=& \pi_{0}(G) \cup
\pi_{\txt{guess}}(G) \cup
\pi_{\txt{sat}}'(G) \cup \\
&\pi_{\txt{guess2}}(G) \cup
\pi_{\txt{sat2}}'(G) \cup \\
&\txt{:- not saturate}
\end{flalign*}
}
\noindent The only difference here is the use the fragment $\pi_{\txt{sat}}'(G)$ as well as the fragment

\begin{theorem}
For every \grappa instance $G$ it holds that $\prf(G) \cong \AS(\pi_\prf(G))$.
\end{theorem}

\end{comment}

\section{System \& overview of experiments}
\label{sec:exper}

We have implemented a system which, given an ADF, generates the encodings for the ADF presented in this work.  The system, \PNAME{YADF} (``Y'' for ``dynamic''), is publicly available (see the link provided in the introduction)
%\footnote{\url{https://www.dbai.tuwien.ac.at/proj/adf/yadf/}}
and currently (version 0.1.1) supports the admissible, complete, preferred, and stable semantics.  It is implemented in \PNAME{Scala} and can, therefore, be run as a \PNAME{Java} executable.

The input format for \PNAME{YADF} is the input format that has become the standard for ADF systems.  Each statement $x$ of the input ADF is encoded via the string $s(x)$ (alternatively, for legacy reasons, also $statement(x)$ can be used).  The acceptance condition $F$ of $x$ is specified in prefix notation via $ac(x,F)$. For example the acceptance conditions of the ADF from Example \ref{ex:adf} are encoded as follows:

\begin{verbatim}
s(a).
s(b).
s(c).
ac(a,or(neg(b),b)).
ac(b,b).
ac(c,imp(c,b)).
\end{verbatim}

\noindent Note the period at the end of each line. Here $or$, $imp$, $neg$ stand for $\vee$, $\rightarrow$, $\neg$ respectively. On the other hand $and$, $c(v)$, and $c(f)$ can be used for $\wedge$, $\top$, and $\bot$.

A typical call of \PNAME{YADF} (using a \PNAME{UNIX} command line) looks as follows:

\begin{verbatim}
java -jar yadf_0.1.1.jar -adm -cred a filename | \ 
      ./path/to/lpopt | ./path/to/clingo
\end{verbatim}

\noindent Here we ask \PNAME{YADF} for the encoding of credulous reasoning \wrt the admissible semantics for the ADF specified in the file specified via $filename$ and the statement $a$.  As hinted at in the introduction to this work, using the rule decomposition tool $\PNAME{lpopt}$\footnote{\url{https://www.dbai.tuwien.ac.at/research/project/lpopt/}}~\cite{BichlerMW16c} is recommended for larger ADF instances.  We have tested \PNAME{YADF} using the ASP solver \PNAME{clingo}~\cite{GebserKKLOORSSW18}.  We provide the complete usage (subject to change in future versions) of \PNAME{YADF}:

\begin{verbatim}
usage: yadf [options] inputfile
with options:
 -h              display this help
 -adm            compute the admissible interpretations
 -com            compute the complete interpretations
 -prf            compute the preferred interpretations
 -stb            compute the stable interpretations
 -cred s         check credulous acceptance of statement s
 -scep s         check sceptical acceptance of statement s
\end{verbatim}

A generator for encodings for \PNAME{GRAPPA} following the methodology outlined in this work are also available as part of the system \PNAME{GrappaVis}\footnote{\url{https://www.dbai.tuwien.ac.at/proj/adf/grappavis/}}.  The focus of this system is on providing graphical means of specifying and evaluating \PNAME{GRAPPA} instances and it includes means of generating static as well as dynamic encodings to ASP.

We reported on an empirical evaluation of the performance of \PNAME{YADF} \wrt the main alternative ADF systems available at that time in \cite{BrewkaDHLW17}.  The other ADF systems we considered then are the static ASP-based system \PNAME{DIAMOND}~\cite{EllmauthalerS2014} (version 0.9) as well as the QBF-based system \PNAME{QADF}~\cite{DBLP:conf/comma/DillerWW14} (version 0.3.2).  In these experiments we found \PNAME{YADF} to be competitive for credulous reasoning under the admissible semantics, while outperforming \PNAME{DIAMOND} and \PNAME{QADF} when carrying out skeptical reasoning under the preferred semantics.
%:  (static ASP based system;  and  . % To generate ADFs, in our study for AAAI'17 we first used a ``grid-based'' ADF generator which has been employed in the then only evaluations of ADF systems known to us~\cite{stefanell12,Diller14,DillerWW14}.  The grid-based generator has the drawback that the underlying grid structure is somewhat arbitrary and the control over the form of the acceptance conditions generated rather limited.  For these reasons,  we also wrote our own graph-based generator\footnote{This generator underlies the subsequent version available at \url{https://www.dbai.tuwien.ac.at/proj/grappa/subadfgen/}.}.

Since our experimental evaluation for AAAI'17, there have been two particularly noteworthy experimental evaluations of ADF systems including our system \PNAME{YADF}~\cite{DillerZLW18,LinsMarNisWalWol18}.  All of these evaluations build on the experimental setup of~\cite{BrewkaDHLW17}; in particular, they also focus on credulous reasoning for the admissible semantics as well as skeptical reasoning for the preferred semantics. Yet, the experiments consider larger sets of (also larger) ADF instances\footnote{The more recent experiments do not consider, on the other hand, ADFs generated via the grid-based generator used in the experiments reported on in~\cite{stefanell12,DBLP:conf/comma/DillerWW14} and that we also considered in~\cite{BrewkaDHLW17}.  Note, nevertheless, that the grid-based generator offers less flexibility in generating ADFs than the graph-based generator used in subsequent experiments.}.  Also, the more recent experimental evaluations include the latest version of \PNAME{DIAMOND}, \PNAME{goDIAMOND}~\cite{StrassE17}, as well as a novel system for ADFs based on incremental SAT solving, \PNAME{k++ADF} (presented in~\cite{LinsMarNisWalWol18}).
%; we have collaborated with two of these evaluations~\cite{Keshavarzi17,DillerZLW18}, while the other has been carried out by colleagues of ours~\cite{LinsMarNisWalWol18}.
We give a brief overview of the setup and results of the mentioned evaluations to then summarise what is currently known about the performance of \PNAME{YADF} vs. other existing ADF systems.
%, which includes those we present in this work.

As already indicated, the experimental evaluations of~\cite{DillerZLW18,LinsMarNisWalWol18} build on the setup from~\cite{BrewkaDHLW17}.  In particular they make use of the graph-based ADF generator we introduced in~\cite{BrewkaDHLW17}\footnote{This generator underlies the subsequent version available at \url{https://www.dbai.tuwien.ac.at/proj/grappa/subadfgen/}.}.  This generator takes any desired directed
graph as input and generates an ADF inheriting the structure of the graph.  This means that the edges of the graph become links and the nodes become statements of the resulting ADF.  In the experiments in~\cite{DillerZLW18,LinsMarNisWalWol18} the generator was only modified to take an undirected rather than directed graph as input (providing more flexibility).  A probability controls whether an edge in the input graph will result in a symmetric link in the ADF (in the experiments a probability of 0.5 is used); in case of non-symmetric links the direction of the link is chosen at random.

For constructing the acceptance conditions of the ADFs, the graph based generator 
assigns each of the parents of a statement to one of 5 different groups (with
equal probability in the experiments).  This assignment determines
 whether the parent participates
in a subformula of the statement's acceptance condition
representing the notions of
attack, group-attack, support, or group-support
familiar from argumentation.
Also, the parents can
appear as literals connected by %
the exclusive-or connective ($\oplus$; this, in order to capture the full complexity of ADFs).  More precisely, if for a statement $s_0$, the parents $s_1,\ldots,s_n$
are assigned to the group for attack, the corresponding subformula for these parents in the acceptance condition of $s_0$ has the form 
%\begin{align*}
 $\neg s_1 \wedge \ldots \wedge \neg s_n$.
%\end{align*}
%\noindent
The subformulas for group-attack, support, group-support, and the exclusive-or group on the other hand have
the form $\neg s_1 \vee \ldots \vee \neg s_n$, $s_1 \vee \ldots \vee s_n$, $s_1 \wedge \ldots \wedge s_n$, and $l_1 \oplus \ldots \oplus l_n$ respectively.  
%\begin{align*}
%\neg s_1 \vee \ldots \vee \neg s_n,\\
%s_1 \vee \ldots \vee s_n,\\
%s_1 \wedge \ldots \wedge s_n,\\
%\end{align*}
%\noindent and
%\begin{align*}
%l_1 \oplus \ldots \oplus l_n
%\end{align*}
%\noindent respectively.
  In the last suformula, $l_i$ ($1 \leq i \leq n$) is either $s_i$ or $\neg s_i$ with equal probability.  Also, for groups to which no parents are assigned, the corresponding subformulas are $\top$ or $\bot$ with identical probability.     
To generate the final acceptance condition of a statement, the subformulas for the different groups of parents of the statement are
connected via $\wedge$ or $\vee$; again, with equal probability.

In~\cite{DillerZLW18} (extending the evaluations from~\cite{Keshavarzi17}) the authors compare the performance of ADF systems on acyclic, being those whose underlying graph is acyclic, vs. non-acyclic ADFs (i.e. ADFs whose underlying
graph is not guaranteed to be acyclic).  The combinations of ADF and back-end systems used are as in~\cite{BrewkaDHLW17}, except that now also the system \PNAME{goDIAMOND} %~\cite{StrassE17}
(version 0.6.6) is considered\footnote{The reason for not using the other more recent versions of \PNAME{DIAMOND}, versions 3.0.x implemented in \PNAME{C++}~\cite{EllmauthalerS2016}, is the decrease of performance to previous versions of \PNAME{DIAMOND} documented in~\cite{StrassE17}. }.  Thus \PNAME{QADF} is version 0.3.2
with
 the preprocessing tool \PNAME{bloqqer} 035~\cite{HeuleJLSB15} %\footnote{\url{http://fmv.jku.at/bloqqer/}}
%\cite{HeuleJLSB15}
 and the QSAT solver 
 \PNAME{DepQBF} 4.0~\cite{LonsingB10,LonsingE17}%\footnote{\url{http://lonsing.github.io/depqbf/}}
 .
 %This version of \PNAME{QADF} includes the (non-QDIMACS) intermediate representation of the QBF encodings as is also included in the version described in Section~\ref{sec::qadf} (version 0.4.0), but does not produce the link-information-sensitive encodings.  
%\cite{DBLP:journals/jsat/LonsingB10}.
%
\PNAME{YADF} is version 0.1.0 with the rule decomposition tool \PNAME{lpopt} version 2.0 %\footnote{\url{https://www.dbai.tuwien.ac.at/research/project/lpopt/}}     
%\cite{BichlerMW16}
%
%
%
%
%
%
 and the ASP solver \PNAME{clingo} 4.4.0~\cite{GebserKKLOORSSW18}%\footnote{\url{https://potassco.org/clingo/}}
 .  Version 0.1.0 of \PNAME{YADF} is identical to version 0.1.1
 %presented in Section~\ref{sec:adf:exps},
  except that it does not generate the encodings for the stable semantics.  The time-out set for the experiments reported on in~\cite{DillerZLW18} is 600 seconds.

The main difference of the experimental setup used in~\cite{DillerZLW18} w.r.t. that of~\cite{BrewkaDHLW17} is that now the benchmark set is generated from Dung argumentation frameworks (AFs) interpreted as (undirected) graphs obtained from benchmarks used at the second international competition of argumentation (ICCMA'17)~\cite{GagglLMW18}.  These result from encoding assumption-based argumentation problems into AFs (``ABA'')~\cite{lehwaljarv17}, encoding planning problems as AFs (``Planning'')~\cite{cergiaval17}, and a data-set of AFs generated from traffic networks (``Traffic'')~\cite{dill17}.
%This last data-set has, in fact, been submitted to ICCMA'17 by us~\cite{dill17} and has been constructed using a larger set of traffic-networks from the same source we used for our evaluations of ADF systems in~\cite{BrewkaDHLW17}.
Specifically, based on preliminary experiments, for the experiments in~\cite{DillerZLW18} 100 AFs were selected at random from a subset of AFs having up to 150 arguments in the very dense AFs in the ``ABA'' data-set, and 100 AFs at random from a subset of AFs having up to 300 arguments in each of the ``Planning'' and ``Traffic'' benchmarks. From the resulting 300 AFs interpreted as undirected graphs, 300 acyclic and 300 non-acyclic ADFs were generated using the graph-based ADF generator.

%Finally, we turn to considering the experiments reported on in~\cite{LinsMarNisWalWol18}.
In the study reported on in~\cite{LinsMarNisWalWol18} the authors generate reasoning problems from the same set of ADFs used in~\cite{DillerZLW18}\footnote{The exact encodings used differ nevertheless because the choice of the statements whose acceptability is checked may differ.} but also consider the new ADF system \PNAME{k++ADF} they implement.
%Concretely, this study incorporates the incremental SAT based system \PNAME{k++ADF}\footnote{\url{https://www.cs.helsinki.fi/group/coreo/k++adf/}} which implements several link information sensitive as well as non-link-information sensitive SAT-based solving procedures for ADFs.  This evaluation strategy is in line with our QBF encodings presented in~\cite{DillerWW14,DillerWW15} (non link information sensitive) and Section~\ref{sec::bqadf}  of this work (link information sensitive) but where e.g. for the link information sensitive procedure the computation of the link types and solving are handled by subsequent calls to a SAT solver rather than in a monolithic encoding (as is the case for our encodings to QBF presented in Section~\ref{sec::bqadf}).
%For the study reported on in~\cite{LinsMarNisWalWol18}
  The authors also make use of novel versions for the back-end systems w.r.t. previous experiments. Thus for \PNAME{goDIAMOND} version 
  0.6.6 is still used but now with \PNAME{clingo} 5.2.1.  For \PNAME{QADF} version 0.3.2\footnote{The paper mistakenly reports use of version 2.9.3 which does not exist (\PNAME{QADF} version 0.3.2 is implemented in version 2.9.3 of the programming language \PNAME{Scala}).}  with \PNAME{bloqqer} 037 and \PNAME{DepQBF} 6.03 is considered. Finally, for \PNAME{YADF} version 0.1.0 is taken in account, but now with \PNAME{lpopt} 2.2  and \PNAME{clingo} 5.2.1.  The version of \PNAME{k++ADF} was version 2018-07-06; the SAT solver used is \PNAME{MiniSAT}~\cite{EenS03} version 2.2.0.  The variation of back-end systems used in~\cite{LinsMarNisWalWol18} w.r.t. previous experiments leads to some variation in the results obtained, in particular for \PNAME{YADF}\footnote{Especially for the admissible semantics.  E.g. \PNAME{YADF} has 47 and 21 time-outs on the ``Traffic'' and ``Planning'' data-sets, while in the study reported on in~\cite{DillerZLW18} the time-outs were 2 and 0 respectively.  We have determined (via tests carried out using the different versions of \PNAME{lpopt} on instances for which there were time-outs in the study from~\cite{LinsMarNisWalWol18}) %; the instances were kindly made available to us by the authors of the latter work)
the cause of this to be the use of \PNAME{lpopt} version 2.2, which seems to have problems in generating the rule-decompositions of some of the encodings obtained via \PNAME{YADF} in a timely-manner; while versions previous to 2.2. (we also tried 2.0 and 2.1) don't have this issue.}.   For the experiments reported on in~\cite{LinsMarNisWalWol18} the time-out is also larger than for previous experiments: 1800 seconds.

  We summarise the results obtained in the studies from~\cite{DillerZLW18} and~\cite{LinsMarNisWalWol18} in bullet-point fashion (we of course refer to the alluded works for details).  As already hinted at, these studies build on (and largely confirm the results obtained in) previous studies, mainly that of~\cite{BrewkaDHLW17}.  When making reference to~\cite{DillerZLW18} we refer to the studies on the (possibly) non-acyclic ADF instances since these can be compared to the study of~\cite{LinsMarNisWalWol18} which does not consider acyclic ADFs (nevertheless, some comments on the performance of especially \PNAME{YADF} on the acyclic instances follow).  In the summary, when mentioning results for a particular solver, we use the best of the results for that solver obtained in the different studies. Also, as a reminder to the reader, in this discussion when we refer to the solvers \PNAME{k++ADF}, \PNAME{goDIAMOND}, \PNAME{YADF}, and \PNAME{QADF}, except if stated otherwise, these are versions 2018-07-06, 0.6.6, 0.1.0, and 0.3.2 respectively.  In particular, we again note that \PNAME{YADF} version 0.1.0 is identical to the newer version detailed in this work (0.1.1) for the encodings considered in the experiments we allude to\footnote{We note also that there is meanwhile a newer version of \PNAME{QADF} (see \url{https://www.dbai.tuwien.ac.at/proj/adf/qadf/}) (version 0.4.0) which includes link-information-sensitive encodings (see Section 3.1.3 of~\cite{dil19}), but is otherwise (i.e. when not using the link-information-sensitive variants of the encodings) identical to version 0.3.2 considered in the experiments reported on in~\cite{DillerZLW18,LinsMarNisWalWol18}.}.

\begin{itemize}
\item For credulous reasoning w.r.t. the admissible semantics each of \PNAME{k++ADF}  (when using the link information sensitive variant \PNAME{ADM-K-BIP}, rather than \PNAME{ADM-2}), \PNAME{goDIAMOND}, and \PNAME{YADF} (making use of \PNAME{lpopt} version 2.0 as in the experiments from~\cite{BrewkaDHLW17,%Keshavarzi17,
  DillerZLW18}), have rather acceptable performance on the ``Traffic'' and ``Planning'' data-sets.  (The same holds for \PNAME{DIAMOND} version 0.9 on a small set of ADFs generated from metro-networks~\cite{BrewkaDHLW17,Keshavarzi17}.)  The order in which we mention the solvers reflects the improvement in performance, with \PNAME{k++ADF} being the clear ``winner''.  The system \PNAME{QADF} (even in the more advantageous configuration with \PNAME{bloqqer} version 035 and \PNAME{DepQBF} version 4.0 from~\cite{BrewkaDHLW17,DillerZLW18}) on the other hand already has quite a few time-outs on the ``Traffic'' and ``Planning'' instances.  We remind the reader that the ``Traffic'' and ``Planning'' data-sets include ADFs with 10 to 300 statements resulting from the underlying graphs obtained from representing transportation-networks as AFs~\cite{dill17} and encoding planning problems into AFs~\cite{cergiaval17} respectively.  
\begin{itemize}
\item Thus \PNAME{k++ADF} (in the link-information-sensitive variant \PNAME{ADM-K-BIP}) had 0 time-outs (1800 seconds) and 0.05 seconds mean running time, \PNAME{goDIAMOND} 0 time-outs and 6.42 seconds mean running time in the experiments reported on in~\cite{LinsMarNisWalWol18} on the ``Traffic'' data-set.  \PNAME{YADF} had 2 time-outs (600 seconds) and 5.68 seconds mean running time (disregarding time-outs) in the experiments reported on in~\cite{DillerZLW18}.  The system \PNAME{k++ADF} (implementing \PNAME{ADM-K-BIP}) had 0 time-outs and 0.14 seconds mean running time, \PNAME{goDIAMOND} 0 time-outs and 6.72 seconds mean running time in the experiments reported on in~\cite{LinsMarNisWalWol18} on the ``Planning'' data-set.  \PNAME{YADF} had 0 time-outs and 13.20 seconds mean running time in the experiments reported on in~\cite{DillerZLW18}.  \PNAME{QADF} had 25 time-outs and 2.15 seconds mean running time on the ``Traffic'' instances and 59 time-outs and 14.63 seconds mean running time on the ``Planning'' instances~\cite{DillerZLW18}.
\end{itemize}
\item For credulous reasoning w.r.t. the admissible semantics, but now on the ``ABA'' data-set; here all ADF systems have some time-outs, yet again the results for \PNAME{k++ADF} are the most promising.  We remind the reader that the ``ABA'' data set consists in 100 very dense ADFs having between 10 to 150 statements resulting from the underlying graphs of encoding problems for assumption-based-argumentation frameworks to AF reasoning problems~\cite{lehwaljarv17}.  
\begin{itemize}
\item Thus \PNAME{k++ADF} (now in the \PNAME{ADM-2} variant) had 12 time-outs (1800 seconds) and mean running time of 16.12 seconds in the experiments of~\cite{LinsMarNisWalWol18}.  Interestingly, for the ``ABA'' data-set \PNAME{QADF} (with \PNAME{bloqqer} version 035 and \PNAME{DepQBF} version 4.0) gets ``second-place'' having 30 time-outs (600 seconds) and 8.15 seconds mean running time in the experiments from~\cite{DillerZLW18}.  The system \PNAME{goDIAMOND} has 52 time-outs and \PNAME{YADF} 56 time-outs in the experiments from~\cite{DillerZLW18}. 
\end{itemize}
\item For the preferred semantics, the performance of \PNAME{YADF} and \PNAME{QADF} (as well as version 0.9 of \PNAME{DIAMOND} on ADFs resulting from traffic networks~\cite{BrewkaDHLW17,Keshavarzi17}) worsens considerably on the ``Traffic'' and ``Planning'' problems (w.r.t. results for the admissible semantics), while \PNAME{k++ADF} (particularly in the link-information-sensitive variant \PNAME{PRF-K-BIB-OPT}, but not in the variant \PNAME{PRF-3}) and \PNAME{goDIAMOND} have much better performance.
\begin{itemize}
\item Thus \PNAME{YADF} (\PNAME{lpopt} version 2.0) has 36 time-outs on the ``Traffic'' instances and 71 time-outs on the ``Planning'' instances in the study from~\cite{DillerZLW18}.  \PNAME{QADF} has 80 and 100 time-outs on the ``Traffic'' and ``Planning'' benchmarks (again, study from~\cite{DillerZLW18}).  On the other hand, \PNAME{goDIAMOND} has 0 time-outs on both data-sets with 28.42 seconds and 17.52 seconds mean running times on the ``Traffic'' and ``Planning'' instances respectively~\cite{LinsMarNisWalWol18}.  The system \PNAME{k++ADF} (in the \PNAME{PRF-K-BIB-OPT} variant) manages having only 1 time-out on the ``Traffic'' instances and 3 on the ``Planning'' instances with 17.18 and 11.14 seconds mean running time respectively~\cite{LinsMarNisWalWol18}.
\end{itemize}
\item All ADF systems also have time-outs when solving skeptical acceptance w.r.t the preferred semantics on the ``ABA'' data-set, with \PNAME{k++ADF} in the \PNAME{PRF-K-BIB-OPT} variant having the least (16 time-outs~\cite{LinsMarNisWalWol18}) and \PNAME{QADF} the most (81 time-outs~\cite{DillerZLW18}).
\begin{itemize}
\item Thus \PNAME{k++ADF} in the \PNAME{PRF-K-BIB-OPT} variant has 16 time-outs and 25.90 seconds mean running time~\cite{LinsMarNisWalWol18}, while \PNAME{QADF} has 81 time-outs (with 32.73 seconds running time on the remaining instances)~\cite{DillerZLW18}.  \PNAME{YADF} has 57 time-outs and 39.46 seconds mean running time, while \PNAME{goDIAMOND} has 52 time-outs and 27.67 seconds mean running time~\cite{DillerZLW18}.     
\end{itemize}
\end{itemize}

To conclude, while our experiments from~\cite{BrewkaDHLW17} (on the instances obtained via the grid-based generator first used in~\cite{stefanell12} and ADFs constructed from a limited set of traffic networks also used in subsequent experiments) suggested \PNAME{YADF} to be the better performing of the then considered systems (including \PNAME{DIAMOND} version 0.9 and \PNAME{QADF} 0.3.2), this picture has changed with subsequent experiments~\cite{%Keshavarzi17,
  DillerZLW18,LinsMarNisWalWol18} involving the new systems \PNAME{goDIAMOND} and \PNAME{k++ADF} as well as more (and larger) data-sets.  In particular, the clearly overall best performing approach for ADF systems seems to be, at current moment, the incremental SAT-based approach implemented in the system \PNAME{k++ADF} (despite the fact that even this system still has quite a few time-outs for the preferred semantics on the ABA data-set).  But even just considering ASP-based systems, while competitive for the admissible semantics, \PNAME{YADF} is clearly behind in performance w.r.t. \PNAME{goDIAMOND} for the preferred semantics on the ``Traffic'' and ``Planning'' data-sets.

Some reason for nevertheless sticking to the dynamic ASP based approach presented in this work (vs. static encodings) is provided by the results on the performance of \PNAME{YADF} on the ABA data-set (in the configurations from~\cite{DillerZLW18}).  Here the constraint built into \PNAME{goDIAMOND} of not supporting ADFs with statements having more than 31 parents is reflected in the constant number of time-outs (52; and similar mean running times: ca. 21 seconds for admissible, 27 seconds for preferred) on all reasoning tasks (admissible and preferred) and for acyclic as well as non-acyclic instances (the latter in the experiments from~\cite{DillerZLW18}).  Indeed, the constraint built in to \PNAME{goDIAMOND} of not supporting ADFs with statements having more than 31 parents is due to the fact that this system (as previous versions of \PNAME{DIAMOND}) needs to convert acceptance conditions of ADFs into a boolean function representation (with a potential exponential explosion), which our dynamic encoding strategy allows to circumvent. Thus, while \PNAME{YADF} still has many time-outs (in fact, a few more than \PNAME{goDIAMOND}) there is some (slight) improvement on the acyclic instances: 54 time-outs with 7.38 seconds mean running time vs. 56 time-outs with 31.39 seconds mean running time for the admissible semantics and 54 time-outs with 16.77 seconds mean running time vs. 57 time-outs with 39.46 seconds mean running time for the preferred semantics.  These results suggest room for improvement as well as, in accordance with the theoretical considerations motivating our dynamic ASP-based approach, a potential niche for the use of (a further optimised) \PNAME{YADF} vs. e.g. \PNAME{goDIAMOND}.

\section{Discussion}\label{sec:disc}

In this work, we developed novel ASP encodings
for advanced reasoning problems in argumentation that
reach up to the third level of the polynomial hierarchy.
Compared to previous work, we rely on translations that
make a single call to an ASP-solver sufficient.
The key idea is to reduce one dimension of complexity to ``long'' rule
bodies, exploiting the fact that checking whether such a rule fires
is already $\NP$-complete (as witnessed by the respective complexity
of conjunctive queries \cite{stoc:ChandraM77}); see also 
\cite{BichlerMW2016} who advocated this idea as a programming technique in the world of ASP.

We implemented our approach for %
ADF and GRAPPA.  %\footnote{The system for ADF reasoning is available at \url{https://www.dbai.tuwien.ac.at/proj/adf/yadf/}.
%The encodings for GRAPPA are available as part of a larger system:
%\url{https://www.dbai.tuwien.ac.at/proj/adf/grappavis/}.}.
  Our experiments show the potential of our approach.  Still,
the number of statements we can handle
is somewhat limited.  Our encodings thus might also be
interesting benchmarks for ASP competitions.  
Nontheless, there are certain aspects which have to be considered in future versions of our system.
In fact, a crucial aspect for the programming technique 
due to \cite{BichlerMW2016} is the possible decomposition of long
rules, since grounders have severe problem  with such rules. As reported in a recent paper~\cite{BichlerMW18}
that also employs this technique, the actual design of long rules can strongly influence the runtime.
We shall thus analyse our encodings in the light of the findings in~\cite{BichlerMW18}
in order to allow for better decomposition whenever possible.

Beyond boosting performance of our system, future work is to apply our approach to alternative ADF semantics~\cite{Polberg14} as well as more recent generalizations of ADFs and GRAPPA such as weighted ADFs~\cite{BrewkaSWW18} %(with finite values).  
and multi-valued GRAPPA~\cite{BrewkaPW19}; for dealing with possibly infinitly many values in this context, 
recent advances in ASP~\cite{JanhunenKOSWS17} might prove useful.
Also the application of other recent ASP techniques (e.g.\ \cite{BogaertsJT16,Redl17})
that allow for circumventing the problem of an exponential blow-up when problems beyond
the second level of the polynomial are treated is of interest.

\bibliographystyle{acmtrans}
\bibliography{ecai2}

\label{lastpage}
\end{document}